\newif\ifijrr
\ijrrfalse
\ifdefined\IJRR\ijrrtrue\fi
\ifdefined\IEEE\ijrrfalse\fi

\ifijrr
  \documentclass[Afour,times,sageh]{sagej}
\else
  \documentclass[letterpaper,10pt,conference]{ieeeconf}
  \IEEEoverridecommandlockouts
\fi

\ifijrr
  \newcommand{\tablecite}[1]{~\citeyearpar{#1}}
\else

  \usepackage{amsthm}
  \usepackage{cite}
  \usepackage{caption}
  \newcommand{\tablecite}[1]{~\cite{#1}}
\fi
\usepackage{amsmath,amssymb,graphicx}
\usepackage{tabularx,multirow,makecell,booktabs}
\usepackage{bm}
\usepackage[table]{xcolor}
\usepackage{url}
\usepackage[english]{babel}
\makeatletter
\let\l@en\l@english
\let\l@EN\l@english
\makeatother
\ifijrr
  \usepackage[hidelinks]{hyperref}
\else
  \usepackage{hyperref}
\fi
\usepackage{cleveref}
\usepackage{mathtools}
\let\labelindent\relax
\usepackage{enumitem}
\usepackage[ruled,vlined]{algorithm2e}
\usepackage{flushend}

\hypersetup{pdftitle={Policy Library CBF: Finite-Horizon Safety at Runtime via Parallel Rollouts},pdfauthor={Taekyung Kim, Hideki Okamoto, Bardh Hoxha, Georgios Fainekos, Dimitra Panagou}}

\newcommand{\vx}{{\boldsymbol x}}
\newcommand{\vu}{{\boldsymbol u}}

\newcommand{\vp}{{\boldsymbol p}}
\newcommand{\vv}{{\boldsymbol v}}

\newcommand{\lieder}{L}
\newcommand{\StateSpace}{\mathcal{X}}
\newcommand{\ControlSpace}{\mathcal{U}}
\newcommand{\RealSpace}{\mathbb{R}}

\newcommand{\calC}{\mathcal{C}}
\newcommand{\calK}{\mathcal{K}}

\newcommand{\calO}{\mathcal{O}}
\newcommand{\calS}{\mathcal{S}}

\DeclareMathAlphabet{\mathmybb}{U}{bbold}{m}{n}
\newcommand{\IndicatorFunction}{\mathmybb{1}}

\DeclareMathOperator*{\argmax}{arg\,max}
\DeclareMathOperator*{\argmin}{arg\,min}

\ifijrr
  \newtheoremstyle{sageplainparens}
{6pt plus 2pt minus 2pt}
{6pt plus 2pt minus 2pt}
{\itshape}
{}
{\bfseries}
{.}
{0.75em}
{\thmname{#1}\thmnumber{ #2}\thmnote{ (#3)}}

\newtheoremstyle{sagedefparens}
{6pt plus 2pt minus 2pt}
{6pt plus 2pt minus 2pt}
{\normalfont}
{}
{\bfseries}
{.}
{0.75em}
{\thmname{#1}\thmnumber{ #2}\thmnote{ (#3)}}

\theoremstyle{sagedefparens}
\newtheorem{definition}{Definition}
\newtheorem{assumption}{Assumption}
\newtheorem{remark}{Remark}
\newtheorem{problem}{Problem}
\newtheorem*{approach}{Approach}
\newtheorem{example}{Example}

\theoremstyle{sageplainparens}
\newtheorem{theorem}{Theorem}
\newtheorem{lemma}{Lemma}
\newtheorem{proposition}{Proposition}

\else
\newtheorem{definition}{Definition}

\newtheorem{theorem}{Theorem}

\theoremstyle{definition}
\newtheorem{remark}{Remark}

\theoremstyle{definition}
\newtheorem{problem}{Problem}

\theoremstyle{definition}

\theoremstyle{definition}

\theoremstyle{definition}

\fi

\ifijrr
  \def\volumeyear{2026}
\def\journalname{The International Journal of Robotics Research}
\title{Policy Library CBF: Finite-Horizon Safety at Runtime via Parallel Rollouts}
\author{Taekyung Kim\affilnum{1}, Hideki Okamoto\affilnum{2}, Bardh Hoxha\affilnum{2}, Georgios Fainekos\affilnum{2}, and Dimitra Panagou\affilnum{1,3}}
\affiliation{
\affilnum{1}Department of Robotics, University of Michigan, Ann Arbor, MI, USA\\
\affilnum{2}Toyota Motor North America, Research \& Development, Ann Arbor, MI, USA\\
\affilnum{3}Department of Aerospace Engineering, University of Michigan, Ann Arbor, MI, USA
}
\corrauth{Taekyung Kim, Department of Robotics, University of Michigan, Ann Arbor, MI 48109, USA.}
\email{taekyung@umich.edu}
\newenvironment{sagebackmatter}[1]
  {\subsection*{\normalsize\sagesf\bfseries #1}\begin{refsize}\noindent}
  {\end{refsize}}
\newcommand{\sagedeclsubhead}[1]{\par\medskip\noindent{\sagesf\bfseries #1}\par\nobreak\noindent}

\else
  \title{\LARGE\bf Policy Library CBF: \\ Finite-Horizon Safety at Runtime via Parallel Rollouts}
\author{Taekyung Kim$^{1}$, Hideki Okamoto$^{2}$, Bardh Hoxha$^{2}$, Georgios Fainekos$^{2}$, Dimitra Panagou$^{1,3}$%
\thanks{$^{1}$Department of Robotics, $^{3}$Department of Aerospace Engineering, University of Michigan, Ann Arbor, MI, 48109, USA {\tt\footnotesize taekyung@umich.edu, dpanagou@umich.edu}}
\thanks{$^{2}$Toyota Motor North America, Research \& Development, Ann Arbor, MI, 48105, USA}
}

\DeclareCaptionFont{mysize}{\fontsize{8}{9.6}\selectfont}
\AddToHook{env/abstract/end}{\unskip\space\href{https://www.taekyung.me/plcbf}{\textcolor{red}{[Project Page]}}\footnote{Project page: \href{https://www.taekyung.me/plcbf}{https://www.taekyung.me/plcbf}} \href{https://plcbf.taekyung.me/}{\textcolor{red}{[Web Demo]}} \href{https://youtu.be/TzeRALHHIdw}{\textcolor{red}{[Video]}} \href{https://github.com/tkkim-robot/plcbf}{\textcolor{red}{[Code]}}
}

\fi

\begin{document}
\ifijrr
  \runninghead{Kim et al.}
  \begin{abstract}
Safety-critical autonomy in unstructured environments poses significant challenges for online safety certification under evolving constraints, particularly when a single backup policy is not effective across operating conditions. We propose the Policy Library Control Barrier Function~(PL-CBF), a runtime safety filter that evaluates a library of fallback policies via parallel finite-horizon rollouts. Among policies with constraint-satisfying rollouts and feasible CBF constraints, PL-CBF selects the policy maximizing the volume of its admissible control set and solves a single quadratic program to minimally modify the nominal input within that constrained input set. Using a finite-horizon language metric over closed-loop behaviors, we characterize when a finite policy library contains a constraint-satisfying rollout. Five simulation studies cover indoor robot navigation with a double-integrator model ($4$ states), highway driving under abrupt friction changes with a nonlinear vehicle model ($8$ states), and quadrotor flight control among dynamic obstacles ($12$ states), demonstrating improved safety performance compared with the evaluated baselines while maintaining millisecond-scale mean computation times. Hardware experiments on a mobile robot among moving pedestrians and a custom-built quadrotor among dynamic obstacles demonstrate real-time operation of PL-CBF.
\end{abstract}

  \maketitle
\else
  \maketitle
  \thispagestyle{empty}
  \pagestyle{empty}
  \begin{abstract}
Safety-critical autonomy in unstructured environments poses significant challenges for online safety certification under evolving constraints, particularly when a single backup policy is not effective across operating conditions. We propose the Policy Library Control Barrier Function~(PL-CBF), a runtime safety filter that evaluates a library of fallback policies via parallel finite-horizon rollouts. Among policies with constraint-satisfying rollouts and feasible CBF constraints, PL-CBF selects the policy maximizing the volume of its admissible control set and solves a single quadratic program to minimally modify the nominal input within that constrained input set. Using a finite-horizon language metric over closed-loop behaviors, we characterize when a finite policy library contains a constraint-satisfying rollout. Five simulation studies cover indoor robot navigation with a double-integrator model ($4$ states), highway driving under abrupt friction changes with a nonlinear vehicle model ($8$ states), and quadrotor flight control among dynamic obstacles ($12$ states), demonstrating improved safety performance compared with the evaluated baselines while maintaining millisecond-scale mean computation times. Hardware experiments on a mobile robot among moving pedestrians and a custom-built quadrotor among dynamic obstacles demonstrate real-time operation of PL-CBF.
\end{abstract}

\fi

\section{INTRODUCTION}

The deployment of autonomous systems in unstructured and
real-world environments represents one of the most significant
engineering challenges, especially under \emph{perception-limited} information, i.e., when the safe area is sensed online~\cite{zhou_raptor_2021,agrawal_gatekeeper_2024,kim_visibilityaware_2025},
dynamic obstacles appear and disappear from
view~\cite{firoozi_oampc_2025,park_collision_2026},
and operating conditions may change over
time~\cite{spielberg_neural_2022,kim_physics_2022,yang_safe_2024}.
Consequently, the robot must continually reassess constraint
satisfaction using the latest perceived model and constraints.

Formal verification tools such as Hamilton-Jacobi~(HJ)
reachability analysis provide strong safety guarantees by
characterizing viable states under specified dynamics and
constraints~\cite{altarovici_general_2013,bansal_hamiltonjacobi_2017,choi_robust_2021}.
However, grid-based approaches face the curse of dimensionality,
which limits their application to high-dimensional robotic systems.
Control Barrier Functions~(CBFs) provide a more scalable
mechanism for enforcing pointwise constraints on the control
input~\cite{ames_control_2019}.
Yet, finding CBFs is difficult in general for nonlinear
systems~\cite{garg_advances_2024}.

\begin{figure*}[t]
    \centering
    \includegraphics[width=0.95\textwidth]{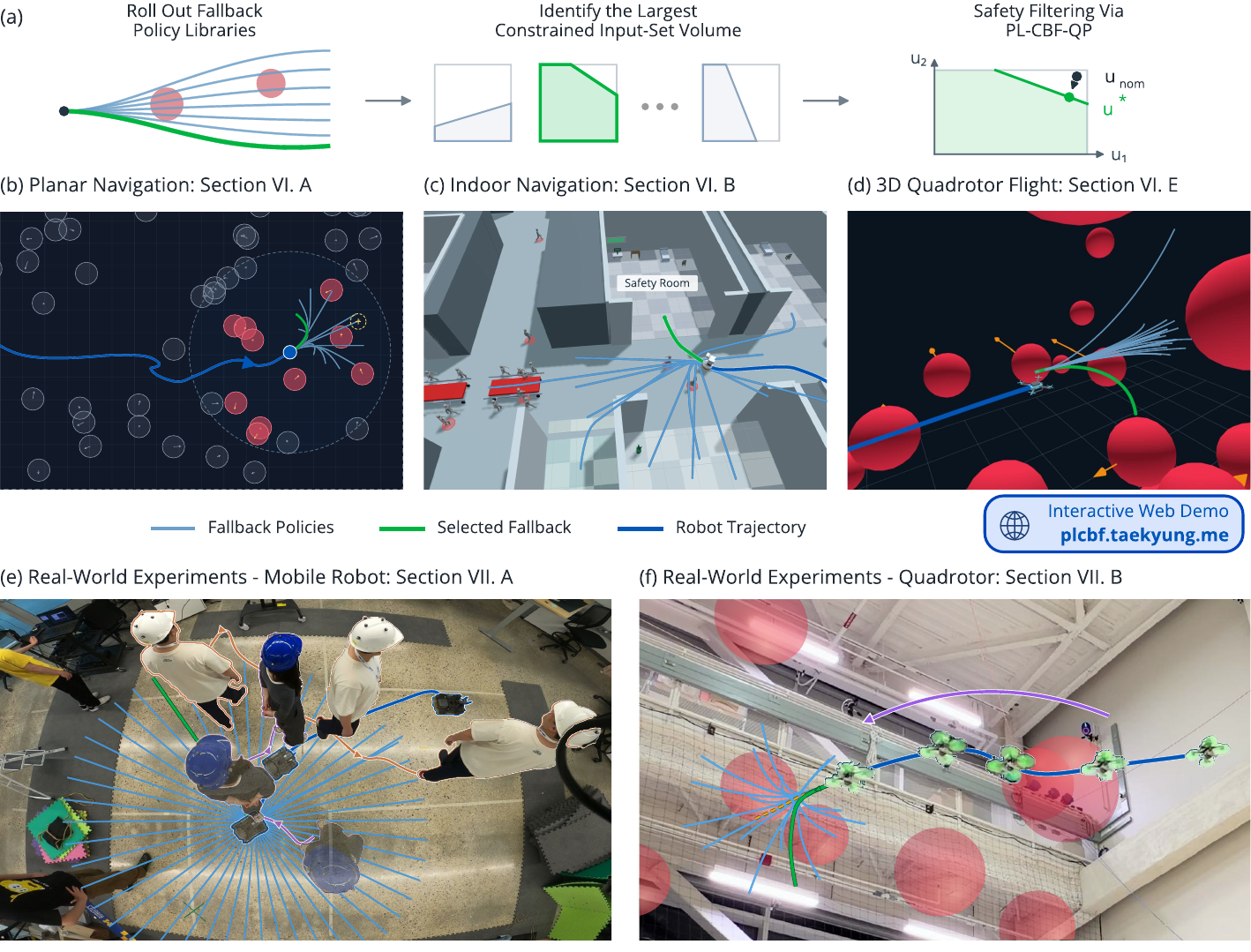}
    \caption{
    Overview of PL-CBF and its evaluation.
    (a) Candidate fallback rollouts define admissible control
    sets; PL-CBF selects the
feasible fallback policy with the largest admissible input volume and filters the nominal input through
    a QP.
    (b)--(d) Simulation examples of planar navigation,
    mobile robot navigation in narrow corridors, and
    3D quadrotor flight, respectively.
    Corresponding live interactive web demos at
    \href{https://plcbf.taekyung.me}{https://plcbf.taekyung.me}
    allow users to add obstacles and change goals.
    (e)--(f) Hardware experiments with a holonomic rover
    among pedestrians and a custom-built quadrotor.
    The red spheres in (f) are virtual dynamic obstacles
    overlaid on the recorded video.
    }
    \label{fig:plcbf_overview}
\end{figure*}

A complementary line of work addresses feasibility by coupling
safety filtering with a certified recovery maneuver.
In particular, Backup CBFs~\cite{chen_backup_2021}
use a designated backup policy and a terminal controlled
invariant set to maintain a feasible route to a region in the
state space where safety can be sustained indefinitely.
Related backup-based approaches include model predictive
shielding~(MPS)~\cite{bastani_safe_2021} and
\texttt{gatekeeper}~\cite{agrawal_gatekeeper_2024}.
Likewise, an invariant terminal set and a compatible controller
are commonly used to establish recursive feasibility in model
predictive control~(MPC)~\cite{borrelli_predictive_2017}.
Constructing and updating these policies and invariant sets
can be difficult for complex nonlinear dynamics, especially
when the operational envelope and constraints depend on online
perception.

Policy value functions also provide a basis for CBF
construction~\cite{so_how_2024}.
More recently, Policy CBFs~(PCBFs) construct safety constraints
at runtime using finite-horizon policy rollouts, avoiding
offline barrier function synthesis~\cite{knoedler_safety_2025}.
However, their effectiveness depends on a designated fallback
policy, which may become unsuitable as constraints are revealed
or operating conditions change.
Prior work has also considered multiple backup strategies
through minimum-intervention selection~\cite{chen_guaranteed_2021},
learned switching~\cite{janwani_learningbased_2024},
and safe-set aggregation~\cite{ong_safespace_2026a}.
These approaches rely on backup safety certificates,
often involving known terminal invariant sets.
However, constructing such certificates remains challenging
when dynamics and constraints change online.

For many robotic systems, constructing
fallback maneuvers that intend to maintain safety over short time horizons can be done using domain knowledge or motion
primitives~\cite{pivtoraiko_differentially_2009}.
We refer to the candidate feedback policies that generate
these maneuvers as \emph{fallback policies}.
These policies are not required to guarantee safety over an
infinite horizon and thus need not be backup policies.
The key idea is to assess their suitability as finite-horizon safety filters through
finite-horizon rollouts under the current model and perceived constraints.
The independence of these rollouts enables parallel
evaluation for computationally efficient runtime
safety filtering.

In this paper, we propose the \emph{Policy Library Control
Barrier Function}~(PL-CBF), a runtime safety filter that
evaluates a library of fallback policies through parallel
finite-horizon rollouts.
Among policies whose predicted trajectories satisfy the current
safety constraints and whose CBF-admissible control sets are
nonempty, PL-CBF selects the policy maximizing the volume of
its admissible control set.
This provides a broader range of admissible control inputs
for adjusting the nominal command.
A single quadratic program~(QP) then minimally modifies
the nominal input within the selected set.
PL-CBF can accommodate both predefined fallback policies
and policies generated online, providing finite-horizon safety
certificates for their rollouts.
For a library of backup policies satisfying
the conditions specified in Section~\ref{sec:guarantees},
PL-CBF additionally guarantees recursive feasibility
and safety of the executed trajectory.
Figure~\ref{fig:plcbf_overview} summarizes the method and its
simulation and hardware evaluation.

The main contributions of this paper are:
\begin{itemize}[leftmargin=*,itemsep=1pt,topsep=2pt]
\item \textbf{Policy-library finite-horizon safety filtering:}
We propose PL-CBF, which evaluates multiple fallback policies
through parallel finite-horizon rollouts to identify policies
whose predicted trajectories satisfy the current state
constraints.
It selects a feasible constraint by admissible input volume
and minimally modifies the nominal input through a single QP.

\item \textbf{Theoretical characterization:}
We derive a policy-library coverage condition using a
finite-horizon language metric and characterize safety
and recursive feasibility under the stated
certified-backup assumptions.

\item \textbf{Simulation and hardware validation:}
Five simulation studies cover planar safe-set analysis,
indoor corridor clearance, highway friction changes, and
quadrotor navigation with linearized and nonlinear dynamics.
Comparisons with single-policy and multiple-backup filters
and a selector ablation demonstrate reduced failures in the
evaluated settings at millisecond-scale mean computation times.
Hardware experiments demonstrate deployment on a holonomic
rover among pedestrians and a custom-built quadrotor among
virtual moving obstacles.
\end{itemize}

\section{BACKGROUND INFORMATION \label{sec:preliminaries}}
\subsection{Control Barrier Function \label{sec:cbf}}

Consider the continuous-time control-affine system
\begin{equation} \label{eq:dynamics}
\dot{\vx}(t) = f(\vx(t)) + g(\vx(t))\vu(t),
\end{equation}
where $\vx(t) \in \StateSpace \subseteq \RealSpace^{n}$ is the state and $\vu(t) \in \ControlSpace \subset \RealSpace^{m}$ is the control input.
For notational compactness, we omit the explicit
time argument when it is clear from context,
writing $\vx$ and $\vu$ in place of $\vx(t)$
and $\vu(t)$, respectively.
We assume $\ControlSpace$ is compact and convex, and that $f: \StateSpace \to \RealSpace^{n}$ and $g: \StateSpace \to \RealSpace^{n \times m}$ are locally Lipschitz continuous. Given a continuously differentiable function $h: \StateSpace \to \RealSpace$, define the safe set
\begin{equation}
\calC \coloneqq \{ \vx \in \StateSpace \mid h(\vx) \geq 0 \}.
\end{equation}
For an admissible, locally Lipschitz feedback policy
$\pi:\StateSpace\to\ControlSpace$ and any initial state
$\vx_0\in\StateSpace$, let the solution of the closed-loop system $\dot{\vx}=f(\vx)+g(\vx)\pi(\vx)$ exist on time interval $[0,\tau]$, and denote $\varphi^{\pi}_{\tau}(\vx_0)$ the state of the closed-loop system initiated at $\vx_0$ after elapsed time $\tau\geq0$. Similarly, $\varphi^{\pi}_{\tau}(\vx(t))$ denotes the predicted state at time $t+\tau$ starting the rollout from the current system state $\vx(t)$. For notational brevity, later in the text we drop the dependence of the argument $\vx(t)$ on $t$, and denote $\varphi^{\pi}_{\tau}(\vx)$ the predicted state $\vx(t+\tau)$ after elapsed time $\tau$ starting from $\vx(t)$.

\begin{definition}[CBF \cite{ames_control_2019}]
A continuously differentiable function $h: \StateSpace \to \RealSpace$ is a CBF on $\calC$ for System~\eqref{eq:dynamics} if there exists an extended class $\calK_{\infty}$ function $\alpha$ such that
\begin{equation}
\sup_{\vu \in \ControlSpace} \left[\lieder_f h(\vx) + \lieder_g h(\vx)\vu\right] \geq -\alpha(h(\vx)) \quad \forall \vx \in \calC .
\label{eq:cbf-condition}
\end{equation}
Here, $\lieder_f h=\nabla h^\top f$ and
$\lieder_g h=\nabla h^\top g$ are Lie derivatives.
\end{definition}

Given a CBF $h$, define the CBF-admissible control set
\begin{equation*}
\begin{aligned}
K_{\textup{cbf}}(\vx) \coloneqq \{ \vu \in \ControlSpace \mid {}&\lieder_{f}h(\vx) + \lieder_{g}h(\vx)\vu\\
&+ \alpha(h(\vx)) \geq 0 \}.
\end{aligned}
\end{equation*}
If $h$ is a CBF, then any locally Lipschitz controller satisfying $\pi(\vx)\in K_{\textup{cbf}}(\vx)$ for all $\vx\in\calC$ renders $\calC$ controlled invariant, equivalently $\vx(0)\in\calC$ implies $\varphi^{\pi}_{t}(\vx(0))\in\calC$ for all $t\ge 0$ \cite{garg_advances_2024}.

\subsection{Policy Control Barrier Function \label{sec:pcbf}}

Hereafter, $h$ denotes a state-constraint function that
does not necessarily satisfy the CBF condition
in~\eqref{eq:cbf-condition}. Thus,
$\calC=\{\vx:h(\vx)\geq0\}$ describes constraint
satisfaction, not an invariant set unless additional
conditions are established.

To evaluate constraint satisfaction under a given policy,
we consider its closed-loop trajectory. Since computing
infinite-horizon rollout values is generally impractical
at runtime, we restrict the evaluation to a finite horizon
$T>0$ and define the rollout value
function~\cite{knoedler_safety_2025}:
\begin{equation}\label{eq:pcbf_finite}
H^{\pi}_{T}(\vx)
:= \inf_{\tau \in [0, T]} h(\varphi_{\tau}^{\pi}(\vx)).
\end{equation}

The safety guarantees associated with a policy rollout
depend on the properties of the policy $\pi$.
We formally distinguish between a \emph{backup policy}
and a general \emph{fallback policy}.
While a fallback policy is merely a candidate control law
whose rollout is evaluated for finite-horizon constraint
satisfaction, a backup policy~$\pi_{\textup{b}}$
guarantees safe recovery from a recoverable set
$\calS\subseteq\calC$ to a known terminal invariant
set $\calS_0\subseteq\calC$.

\begin{definition}[Backup Policy]\label{def:backup_policy}

Let $\calS_0\subseteq\calC$ be a known terminal set
and $T>0$ a prescribed backup horizon.
A locally Lipschitz policy
$\pi_{\textup{b}}$ is called a \emph{backup policy}
for $\calS_0$ if:

(i) $\calS_0$ is forward invariant under
$\pi_{\textup{b}}$.

(ii) The recoverable set induced by $\pi_{\textup{b}}$,
defined as
\begin{equation}
\begin{aligned}
\calS := \bigl\{\vx\in\calC \,\big|\,
&\varphi_{\tau}^{\pi_{\textup{b}}}(\vx)\in\calC,
\ \forall\tau\in[0,T],\\
&\varphi_{T}^{\pi_{\textup{b}}}(\vx)
\in\calS_0\bigr\},
\end{aligned}
\end{equation}
contains a neighborhood of $\calS_0$.

\end{definition}

By construction, from any state $\vx\in\calS$, the
trajectory under $\pi_{\textup{b}}$ remains in $\calC$
over $[0,T]$ and satisfies
$\varphi_T^{\pi_{\textup{b}}}(\vx)\in\calS_0$.
Forward invariance of $\calS_0$ under $\pi_{\textup{b}}$
then keeps the trajectory in $\calS_0$ for all
$\tau\geq T$.
This property underlies the infinite-horizon safety
guarantees of backup-based approaches.

However, in environments where the safe set is revealed or
updated online, constructing a valid terminal invariant set
$\calS_0$ \textit{a priori} is difficult.
The PCBF framework therefore considers general fallback
policies without enforcing terminal invariant set
constraints~\cite{knoedler_safety_2025}.
Their rollouts are evaluated for constraint satisfaction
over $[0,T]$; this finite-horizon evaluation alone does
not establish safety for all future time.

\textbf{Limitation:} The limitation of the standard PCBF is its reliance on the validity of \textbf{\emph{one}} specific policy $\pi$.
The safety filter enforces constraints based solely on the closed-loop behavior of this single policy. If this policy becomes unsuitable under changing
operating conditions or newly perceived constraints,
its rollout may violate the state constraints even
when an alternative admissible policy can maintain
safety. This creates unnecessary conservatism and may lead to safety violations that could have been avoided by a different feasible policy.

\section{PROBLEM FORMULATION \label{sec:problem}}

\subsection{Discrete Perception Updates and Perceived Safe Set}
Let $\{t_k\}_{k\in\mathbb{N}}$ denote the discrete sensing-update times. We define the index map
\begin{equation}
\kappa:\RealSpace_{\ge 0}\to\mathbb{N}, \quad \kappa(t)=k \ \text{such that}\ t_k \le t < t_{k+1},
\end{equation}
which returns the index of the latest sensing update available at continuous time $t$.

At each update time $t_k$, the robot constructs a
\emph{perceived safe set} $\calC_k\subseteq\StateSpace$
from the latest environment estimate
(e.g., LiDAR-based obstacle detections and occupancy
mapping), accounting jointly for all perceived
obstacle and environment constraints.
We assume that this set has been conservatively
constructed to account for sensing and estimation
uncertainty, using methods such as certified mapping
or statistical calibration
~\cite{axelrod_provably_2018, agrawal_certifiablycorrect_2025, mei_perceive_2026}.
The design and certification of these perception
modules are outside the scope of this work, and
our analysis is conditional on the validity of
the resulting constraints.
We represent this set as the $0$-superlevel set of
a locally Lipschitz constraint function
$h_k:\StateSpace\to\RealSpace$:
\begin{equation}
\calC_k
=\{\vx\in\StateSpace\mid h_k(\vx)\geq0\}.
\end{equation}
We do not impose a monotonicity assumption on the sequence $\{\calC_k\}_{k \in \mathbb{N}}$: newly observed free space may enlarge the certified region, whereas newly detected obstacles or occlusions may invalidate regions previously regarded as safe.

Using the same notation $\varphi^{\pi}_{(\cdot)}(\cdot)$ introduced in \autoref{sec:cbf}, the look-ahead state at time $t+\tau$ from the current state $\vx(t)$ under policy $\pi$ is $\varphi^{\pi}_{\tau}(\vx(t))$, for all $\tau\ge 0$.

\subsection{Problem Statement}

Our primary objective is to design a safety filter that tracks a desired nominal policy $\pi_{\textup{nom}}$ while ensuring that the system remains within the perceived safe set. Because safety must be certified online using only the currently
perceived safe set~$\calC_k$, where $k=\kappa(t)$ indexes the
current sensing interval, infinite-horizon safety generally
cannot be certified without strong assumptions about future
environmental conditions. We therefore consider a finite-horizon safety problem over a planning horizon $T$.

\begin{problem}[Finite-Horizon Safety]\label{prob:finite_safety}
At time $t$, let $k=\kappa(t)$. Given the current perceived safe set $\calC_k$, synthesize a control policy $\pi$ close to $\pi_{\textup{nom}}$ such that
\begin{equation}
\varphi_{\tau}^{\pi}\big(\vx(t)\big) \in \calC_k, \quad \forall \tau \in [0, T].
\end{equation}
\end{problem}
At each sensing update, \autoref{prob:finite_safety}
is re-evaluated using the latest perceived safe set
$\calC_k$, allowing newly observed constraints to be
incorporated online. However, satisfying the safety
condition for the current snapshot does not guarantee
recursive feasibility under subsequent perception updates.

\begin{remark}[Connections to Model Predictive Control (MPC)]
The finite-horizon safety formulation in \autoref{prob:finite_safety} is closely related to MPC practices. A standard way to establish recursive feasibility in finite-horizon MPC is to enforce a terminal constraint set. However, in environments where such a terminal set is difficult to synthesize online, practitioners often omit terminal constraints as the resulting MPC remains highly effective in practice~\cite{camacho_model_2004, re_automotive_2010}.
Our formulation, which does not certify invariance for all future time, is analogous to dropping the terminal constraint in MPC. However, a key distinction lies in the computational objective. While MPC typically requires solving a computationally expensive optimization problem online,
our goal is to design a computationally efficient
safety filter that constructs a safety certificate
at runtime with millisecond-scale computation times.
\end{remark}

\section{POLICY LIBRARY CBF \label{sec:method}}
\subsection{The Policy Library CBF Formulation}

While rollout-based safety filtering enables runtime
construction of safety constraints, its effectiveness
can be limited by the choice of the policy used for
rollout evaluation. We address this limitation through
the Policy Library Control Barrier Function~(PL-CBF),
which considers a library of $K$ candidate fallback
policies for finite-horizon safety evaluation:
\begin{equation}
\Pi = \{ \pi_1, \pi_2, \dots, \pi_K \}.
\end{equation}

We then define the finite-horizon \emph{policy library value function}\footnote{For the remainder of this paper, we omit the subscript $k$ and denote the current perception snapshot $h_k$ simply by $h$ when the meaning is clear.}
\begin{equation}\label{eq:PL-CBF_value_finite}
\begin{aligned}
H_T^{\Pi}(\vx) & \coloneqq \max_{\pi_i \in \Pi} \left\{ H_T^{\pi_i}(\vx) \right\} \\
& = \max_{\pi_i \in \Pi} \left\{ \inf_{\tau \in [0,T]} h\left(\varphi_{\tau}^{\pi_i}(\vx) \right) \right\}.
\end{aligned}
\end{equation}
Conceptually, if any policy in the library~$\Pi$ can
maintain nonnegative constraint values over $[0,T]$,
then $\vx$ is considered safe over this horizon under
that policy. This induces the \emph{PL-CBF safe set}
\begin{equation}\label{eq:PL-CBF_safeset}
\calS_{T}^{\Pi} \coloneqq \bigcup_{\pi_i \in \Pi} \{ \vx \in \StateSpace \mid H^{\pi_i}_{T}(\vx) \geq 0 \}.
\end{equation}
We emphasize that the explicit construction of $\calS_{T}^{\Pi}$ is used solely for theoretical analysis to characterize how the certified safe set evolves as the policy library $\Pi$ changes. In the runtime algorithm, we \textbf{\emph{do not}} compute the boundary of $\calS_{T}^{\Pi}$ or perform operations over the entire state space $\StateSpace$. At runtime, we only evaluate $H^{\Pi}_{T}(\vx)$ at the current system state $\vx$.
Section~\ref{sec:guarantees} distinguishes finite-horizon
rollout certificates from recursive feasibility and states
the additional conditions used for libraries of backup policies.

\subsection{Safety Filter Synthesis}

To construct the QP-based safety filter, we associate
each candidate policy $\pi_i$ with the following
CBF-admissible control set:
\begin{equation}\label{eq:policy_admissible_set}
\begin{aligned}
K_i(\vx)&\coloneqq K_{\textup{cbf}}(\vx;\pi_i)\\
&=\{\vu\in\ControlSpace:a_i(\vx)^\top\vu\geq b_i(\vx)\},\\
a_i(\vx)&=g(\vx)^\top\nabla H_T^{\pi_i}(\vx),\\
b_i(\vx)&=-\nabla H_T^{\pi_i}(\vx)^\top f(\vx)
-\alpha_i(H_T^{\pi_i}(\vx)).
\end{aligned}
\end{equation}
Each $\alpha_i$ is a locally Lipschitz extended class-$\calK$ function.

We first identify policies with positive rollout
constraint margins and then retain those with
nonempty CBF-admissible control sets:
\begin{equation}\label{eq:policy_eligibility}
\begin{aligned}
\mathcal I_{\textup{safe}}(\vx)
&=\{i:H_T^{\pi_i}(\vx)>0\},\\
\mathcal I_{\textup{feas}}(\vx)
&=\{i\in\mathcal I_{\textup{safe}}(\vx):
K_i(\vx)\neq\varnothing\}.
\end{aligned}
\end{equation}
When $\mathcal I_{\textup{feas}}(\vx)\neq\varnothing$,
PL-CBF selects
\begin{equation}\label{eq:volume_selection}
i^\star\in\argmax_{i\in\mathcal I_{\textup{feas}}(\vx)}
\operatorname{Vol}(K_i(\vx))
\end{equation}
and computes the filtered control input by solving
\begin{equation}\label{eq:mp_cbf_qp}
\vu^\star=\argmin_{\vu\in K_{i^\star}(\vx)}
\|\vu-\vu_{\textup{nom}}\|_W^2.
\end{equation}
Here, $\vu_{\textup{nom}}=\pi_{\textup{nom}}(\vx)$ and $W\succ0$ is a weight matrix.
$\operatorname{Vol}$ denotes $m$-dimensional
Lebesgue volume. The volume criterion favors a larger
admissible control set, while the QP minimizes deviation from the nominal input
within the selected set. If $\mathcal I_{\textup{feas}}(\vx)=\varnothing$,
PL-CBF invokes the contingency control described below. Algorithm~\ref{alg:plcbf} summarizes the overall
PL-CBF procedure.

\begin{algorithm}[t]\footnotesize
\caption{Policy Library CBF QP}
\label{alg:plcbf}
\SetAlgoLined

\KwIn{Current state $\vx$, perceived constraint function
$h$, nominal input $\vu_{\textup{nom}}$, policy library $\Pi$}
\KwOut{
Filtered control input $\vu^\star$}
\BlankLine

\For{$\pi_i\in\Pi$}{
    \tcp{Parallelizable across policies}
    Roll out $\pi_i$ to obtain closed-loop trajectory
    $\varphi_{(\cdot)}^{\pi_i}(\vx)$ \\
    Compute
    $H_T^{\pi_i}(\vx)
    =\inf_{\tau\in[0,T]}
    h(\varphi_\tau^{\pi_i}(\vx))$ \\
    Compute $\nabla H_T^{\pi_i}(\vx)$
    using \texttt{autodiff}\\
    Construct $K_i(\vx)$
    using~\eqref{eq:policy_admissible_set} \\
}
\BlankLine

Determine $\mathcal I_{\textup{safe}}(\vx)$ and
$\mathcal I_{\textup{feas}}(\vx)$
using~\eqref{eq:policy_eligibility} \\

\eIf{$\mathcal I_{\textup{feas}}(\vx)\neq\varnothing$}{
    Select
    $i^\star\in
    \argmax_{i\in\mathcal I_{\textup{feas}}(\vx)}
    \operatorname{Vol}(K_i(\vx))$ \\
    $\vu^\star\leftarrow$
    solve QP~\eqref{eq:mp_cbf_qp}
    using $K_{i^\star}(\vx)$ \\
}{
    $\vu^\star\leftarrow$ contingency control \\
}

\Return{$\vu^\star$} \\

\end{algorithm}

\textbf{Time-Varying Constraints:}
For time-varying constraints, we explicitly retain
the time dependence, which is otherwise omitted
when clear from context. Using the obstacle-motion
prediction available at sensing update $k$,
we write $h_k(\vx,t)$.
The corresponding rollout value function is
\begin{equation*}
H_T^{\pi_i}(\vx,t)
:=\inf_{\tau\in[0,T]}
h_k(\varphi_\tau^{\pi_i}(\vx),t+\tau).
\end{equation*}
In this case, the right-hand side $b_i$ in
\eqref{eq:policy_admissible_set} becomes
\begin{equation}\label{eq:moving_obstacle_rhs}
\begin{aligned}
b_i(\vx,t)
&=-\nabla_{\vx}H_T^{\pi_i}(\vx,t)^\top f(\vx)\\
&\quad-\partial_t H_T^{\pi_i}(\vx,t)
-\alpha_i(H_T^{\pi_i}(\vx,t)).
\end{aligned}
\end{equation}
Here, $\partial_t H_T^{\pi_i}$ accounts for the
time variation of the predicted constraints,
while $a_i$ and $K_i$ also become time-dependent.
For time-invariant constraints, the additional
time-derivative term vanishes.

\textbf{Constraint Selection and Contingency Control:}
When a feasible candidate is available, the selected
policy determines the active CBF constraint by \eqref{eq:volume_selection}, and
the applied control input is computed by the QP
in~\eqref{eq:mp_cbf_qp}.
The active constraint is reselected at subsequent
control updates.
When $\mathcal I_{\textup{feas}}(\vx)=\varnothing$,
the controller directly applies the input of the
policy in $\Pi$ with the largest rollout constraint
margin $H_T^{\pi_i}(\vx)$.
This provides a best-effort contingency response,
although recursive constraint satisfaction is not guaranteed.

\textbf{Volume Computation:}
For box-constrained inputs
$\ControlSpace=\prod_{j=1}^{m}[\underline{\vu}_j,\bar{\vu}_j]$,
each $K_i(\vx)$ is the intersection of
$\ControlSpace$ and a halfspace. Its volume is
\begin{equation}\label{eq:box_halfspace_volume}
\operatorname{Vol}(K_i(\vx))
=
\int_{\ControlSpace}
\IndicatorFunction \left(
a_i(\vx)^\top\vu\geq b_i(\vx)
\right)
\,\mathrm{d}\vu,
\end{equation}
where $\IndicatorFunction$ denotes the indicator function.
Rescaling the box to the unit hypercube allows this volume
to be computed from the closed-form hypercube--halfspace
intersection formula in~\cite[Thm.~1]{marichal_slices_2008},
multiplied by the volume of the original box.

\subsection{Parallel Computation and Runtime Efficiency}

A significant advantage of the proposed PL-CBF
framework is its computational structure, which is
naturally amenable to parallelization because it
relies on independent trajectory rollouts. The closed-loop trajectory
$\varphi_{(\cdot)}^{\pi_i}(\vx)$ and policy value function
$H^{\pi_i}_{T}(\vx)$ for each policy $\pi_i \in \Pi$
can be computed independently of the other candidates
and evaluated in parallel on hardware such as GPUs.
This parallel architecture allows the safety filter
to scale efficiently with the size of the policy
library $K$.

In practice, the library can consist of feedback
policies constructed from state-lattice motion
primitives~\cite{pivtoraiko_differentially_2009}
or maneuvers designed using domain knowledge,
such as braking and evasive turns.
These complementary maneuvers allow PL-CBF to adapt
its choice of CBF constraint to the current model
and perceived constraints, with a single QP solved
after policy selection.

\section{THEORETICAL ANALYSIS \label{sec:analysis}}
\subsection{Language Metric and Policy Library Coverage}

We quantify similarity between finite-horizon
closed-loop trajectories induced by different fallback
policies using a language
metric from the verification literature on metric transition systems~\cite{girard_approximation_2007,girard_approximate_2011}.
This distance characterizes policy-library coverage
for the safety analysis in~\autoref{sec:completeness}.

The finite-horizon trajectory is $\tau\mapsto \varphi_{\tau}^{\pi}(\vx)$ for $\tau\in[0,T]$. In the general framework, the \emph{language} of a system is the set of external trajectories it can generate, and \emph{language metrics} are Hausdorff distances between such trajectory sets induced by a trajectory metric~\cite[Def.~2]{girard_approximate_2011}.

\begin{definition}[Finite-Horizon Language Metric \footnote{A closely related finite-horizon construction for timed languages is referred to as a \emph{timed language metric} in~\cite{deng_verification_2016}.}]
Given an initial condition $\vx$ and a horizon $T$, we define the finite-horizon language metric between two policies $\pi_i$ and $\pi_j$ by
\begin{equation}
d_{\vx}(\pi_i,\pi_j) \coloneqq \sup_{\tau \in [0,T]} \, \|\varphi_{\tau}^{\pi_i}(\vx)-\varphi_{\tau}^{\pi_j}(\vx)\|_2 .
\end{equation}
\end{definition}

The finite-horizon language metric $d_{\vx}$ measures
the maximum state deviation between two policy
rollouts over $[0,T]$, enabling the transfer of
finite-horizon safety margins through the Lipschitz
continuity of $h$ (see \autoref{sec:completeness}).

\begin{definition}[
Policy-Library Approximate Precision]\label{def:approximation_precision}
Let $\Pi_{\textup{all}}$ denote the admissible policies
whose closed-loop trajectories under~\eqref{eq:dynamics}
exist uniquely on $[0,T]$, and let
$\Pi\subseteq\Pi_{\textup{all}}$ be a nonempty finite library.

The approximate precision of $\Pi$ at $\vx$ is
the directed Hausdorff distance from
$\Pi_{\textup{all}}$ to $\Pi$ induced by $d_{\vx}$:

\begin{equation}
\delta_{\vx}(\Pi) \coloneqq \sup_{\pi^{\star} \in \Pi_{\textup{all}}} \, \left[ \min_{\pi_k \in \Pi} \, d_{\vx}(\pi^{\star}, \pi_k) \right].
\end{equation}
\end{definition}

Intuitively, for any admissible policy $\pi^{\star} \in \Pi_{\textup{all}}$, there exists a policy $\pi_k \in \Pi$ whose closed-loop trajectory stays within $\delta_{\vx}(\Pi)$, measured by $d_{\vx}$, of the trajectory induced by $\pi^{\star}$ over the horizon $[0,T]$. Therefore, smaller $\delta_{\vx}(\Pi)$ corresponds to finer coverage of admissible closed-loop behaviors by the finite policy library.

\subsection{Completeness \label{sec:completeness}}

As perceived constraints change, a previously
safe fallback policy may become unsuitable even when a
safe alternative remains achievable.
We now use the coverage measure $\delta_{\vx}(\Pi)$ to
establish a sufficient condition for the library to
contain such an alternative under the current model
and constraints.
This result can be interpreted as a \emph{completeness}
guarantee: if an admissible policy has a strictly
positive safety margin over $[0,T]$ and
$\delta_{\vx}(\Pi)$ is sufficiently small, then the
library contains a policy with a strictly
constraint-satisfying rollout over that horizon.

We define the \textit{clearance} of a trajectory as its
minimum safety margin:
\begin{definition}[Trajectory Clearance]
The clearance $\gamma(\sigma)$ of a trajectory $\sigma$ over $[0,T]$ with respect to the constraint function $h$ is
\begin{equation}
\gamma(\sigma) \coloneqq \inf_{\tau \in [0, T]} h(\sigma(\tau)).
\end{equation}
In particular, $\gamma(\varphi^{\pi}_{(\cdot)}(\vx)) = H^{\pi}_{T}(\vx)$ by~\eqref{eq:PL-CBF_value_finite}.
\end{definition}

\begin{theorem}[Completeness of a Finite-Horizon Safe Policy]
\label{thm:completeness}
Consider a system with a constraint function $h: \StateSpace \to \RealSpace$ that is $L_h$-Lipschitz.
Let $\Pi$ be a finite policy library with approximation error $\delta_{\vx}(\Pi)$ under the finite-horizon language metric $d_{\vx}(\cdot,\cdot)$.
Suppose there exists a safe policy $\pi^{\star} \in \Pi_{\textup{all}}$ that renders the system strictly safe from $\vx$ over $[0,T]$ with clearance $\gamma^{\star} = \gamma(\varphi^{\pi^{\star}}_{(\cdot)}(\vx)) > 0$. If
\begin{equation}
\delta_{\vx}(\Pi) < \gamma^{\star}/L_h ,
\end{equation}
then there exists $\pi_k \in \Pi$ such that $H_{T}^{\pi_k}(\vx) > 0$.
\end{theorem}

\begin{proof}
By the definition of $\delta_{\vx}(\Pi)$, there exists a policy $\pi_k \in \Pi$ such that
\begin{equation*}
d_{\vx}(\pi^{\star},\pi_k)
= \sup_{\tau \in [0, T]} \|\varphi_{\tau}^{\pi^{\star}}(\vx) - \varphi_{\tau}^{\pi_k}(\vx)\|_2
\leq \delta_{\vx}(\Pi).
\end{equation*}
Hence, for all $\tau \in [0,T]$, $\|\varphi_{\tau}^{\pi^{\star}}(\vx) - \varphi_{\tau}^{\pi_k}(\vx)\|_2 \le \delta_{\vx}(\Pi)$. Since $h$ is $L_h$-Lipschitz, for any $\tau \in [0,T]$,
\begin{equation*}
\begin{aligned}
& \big|h(\varphi_{\tau}^{\pi^{\star}}(\vx)) - h(\varphi_{\tau}^{\pi_k}(\vx))\big| \\
\le \quad &  L_h \, \|\varphi_{\tau}^{\pi^{\star}}(\vx) - \varphi_{\tau}^{\pi_k}(\vx)\|_2 \le L_h \, \delta_{\vx}(\Pi).
\end{aligned}
\end{equation*}
Rearranging yields, for all $\tau \in [0,T]$,
\begin{equation}\label{eq:thm1_eq1}
h(\varphi_{\tau}^{\pi_k}(\vx)) \ge h(\varphi_{\tau}^{\pi^{\star}}(\vx)) - L_h \delta_{\vx}(\Pi).
\end{equation}
Taking the infimum over $\tau \in [0,T]$ in \eqref{eq:thm1_eq1} yields
\begin{equation*}
\begin{aligned}
& H_{T}^{\pi_k}(\vx) = \inf_{\tau \in [0,T]} h(\varphi_{\tau}^{\pi_k}(\vx)) \\
& \ge \inf_{\tau \in [0,T]} h(\varphi_{\tau}^{\pi^{\star}}(\vx)) - L_h \delta_{\vx}(\Pi) = \gamma^{\star} - L_h \delta_{\vx}(\Pi).
\end{aligned}
\end{equation*}
If $\delta_{\vx}(\Pi) < \gamma^{\star}/L_h$, then $\gamma^{\star} - L_h \delta_{\vx}(\Pi) > 0$, and therefore $H_{T}^{\pi_k}(\vx) > 0$.
\end{proof}

\textbf{Interpretation:}
Theorem~\ref{thm:completeness} relates library coverage
to the existence of a constraint-satisfying rollout.
The quantity $\delta_{\vx}(\Pi)$ describes this
coverage and is not evaluated by the runtime filter,
which directly computes candidate rollout values
and checks the feasibility of their CBF constraints.
A richer library can reduce the distance to a safe
reference rollout, while smaller clearance requires
finer coverage.

\subsection{Scope of the Safety Guarantees}\label{sec:guarantees}
Theorem~\ref{thm:completeness} provides a condition for the
existence of a constraint-satisfying rollout. To clarify
the implications for the executed controller, we distinguish
the following two cases.

\textbf{(a) General fallback policies:}
For a general fallback policy,
$H_T^{\pi_i}(\vx)\geq0$ certifies constraint satisfaction
over $[0,T]$ when $\pi_i$ is executed under the closed-loop dynamics.
This rollout certificate does not by itself establish
recursive feasibility of the safety-filter QP.
Additional conditions are required, as in the following case.

\textbf{(b) Backup policies:}
Backup policies and their terminal invariant sets can be constructed in settings that admit
analytically verifiable recovery maneuvers,
such as braking or obstacle evasion for
low-dimensional motion models with known
constraints~\cite{chen_backup_2021,chen_guaranteed_2021,
janwani_learningbased_2024}.
Restricting $\Pi$ to such backup policies whose rollout
functions $H_T^{\pi_i}$ are valid CBFs for the chosen
$\alpha_i$ gives a special case of PL-CBF.
For fixed dynamics and constraints, starting from
$H_T^\Pi(\vx(0))>0$, the selected CBF inequality
keeps the corresponding rollout value positive.
The previously selected policy therefore remains
eligible, and its CBF constraint remains feasible
at the next selection.
Repeating this argument yields recursive feasibility and,
since $H_T^{\pi_i}(\vx)\leq h(\vx)$,
$h(\vx(t))\geq0$ for as long as the closed-loop
solution exists.

For $K=1$, the optimization reduces to the corresponding
single policy CBF-QP.

\subsection{Inclusion of Nominal Policy}

Many backup-based safety filters exhibit unnecessary conservatism by modifying the nominal control even when the system is operating safely~\cite{chen_backup_2021, so_how_2024, knoedler_safety_2025, bastani_safe_2021}.
This phenomenon, which we refer to as
``safety evaluation on backup''~\cite{kim_backupbased_2026},
arises because safety is evaluated based on the
feasibility of a specific backup maneuver rather
than the nominal trajectory itself. If the backup policy is inherently conservative or currently ill-suited to the environment, the filter may restrict the nominal controller unnecessarily to ensure the backup remains feasible.

Including $\pi_{\textup{nom}}$ in $\Pi$ allows PL-CBF
to evaluate the nominal trajectory directly alongside
the predefined fallback maneuvers.
This is particularly useful when the nominal trajectory
satisfies the constraints while the predefined fallback
maneuvers are overly restrictive.
Whenever the selected CBF-admissible set
$K_{i^\star}(\vx)$ contains the nominal input
$\vu_{\textup{nom}}=\pi_{\textup{nom}}(\vx)$,
this input is feasible for the QP
in~\eqref{eq:mp_cbf_qp} and minimizes its objective.
Consequently, $\vu^\star=\vu_{\textup{nom}}$,
yielding zero intervention.

\section{SIMULATION RESULTS \label{sec:results}}

We evaluate PL-CBF against representative backup-based
safety filters: Model Predictive Shielding
(MPS)~\cite{bastani_safe_2021},
\texttt{gatekeeper}~\cite{agrawal_gatekeeper_2024},
and Backup CBF~\cite{chen_backup_2021}.
We also compare against MB-CBF~\cite{chen_guaranteed_2021},
which employs multiple backup policies, and
PL-CBF-MI, an ablation of PL-CBF that replaces
admissible-input volume selection with
minimum-intervention selection, as detailed below.
One of our case studies includes mixed-integer model
predictive control (MI-MPC), which uses binary
variables to select among candidate maneuver
rollouts~\cite{borrelli_predictive_2017}. PCBF~\cite{knoedler_safety_2025} is also included
as a single-fallback ablation of the proposed
policy-library formulation, corresponding to
$|\Pi|=1$.

Synthesizing and certifying backup policies online is
challenging, particularly for nonlinear systems under
limited sensing and changing dynamics or constraints.
We design our experiments to evaluate the ability of
safety filters to maintain safety under changes in
the dynamics and constraints.
For each experiment, we first construct backup policies
based on the initial dynamics and constraints and verify
offline that they are valid backup policies under
these conditions.
However, during execution, changes in the dynamics,
newly perceived obstacles, or obstacle motion can
invalidate these initially feasible maneuvers.
This is particularly relevant to backup-based safety
filters, whose guarantees depend on the validity of
the supplied backup policies and their associated
certificates.

\textbf{Multiple backup baseline:}
MB-CBF follows the multiple backup CBF selection
rule in~\cite{chen_guaranteed_2021}.
It solves a Backup-CBF QP for each backup policy
and selects the feasible solution that minimizes
the deviation from the nominal control input:
\begin{equation}\label{eq:multi_backup_mi}
i_{\textup{MB}} \in
\argmin_{i\in\mathcal I_{\textup{feas}}}
\|\vu_i^\star-\vu_{\textup{nom}}\|_W^2,
\end{equation}
where $\vu_i^\star$ is the standard Backup CBF-QP solution for policy
$\pi_i$, $\mathcal I_{\textup{feas}}$ denotes the
indices of feasible QPs, and $W\succ0$ is a
weight matrix.
The safety filter then applies
$\vu_{i_{\textup{MB}}}^\star$.
This requires solving $|\Pi|$ Backup-CBF QPs
per control update, increasing computational cost
with the library size.

\textbf{Policy library ablation:}
PL-CBF-MI is an ablation of PL-CBF that replaces
the admissible input volume selection rule with
minimum intervention selection.
It solves a QP for each fallback policy with a positive
rollout margin, using the corresponding admissible-input set
$K_i(\vx)$ in~\eqref{eq:policy_admissible_set}.
Let $\vu_i^\star$ denote the resulting QP solution for each $\pi_i$.
The selected policy index is
\begin{equation}\label{eq:library_mi}
i_{\textup{MI}} \in
\argmin_{i\in\mathcal I_{\textup{feas}}}
\|\vu_i^\star-\vu_{\textup{nom}}\|_W^2,
\end{equation}
where $\mathcal I_{\textup{feas}}$ denotes the
indices of feasible QPs.
Only policies satisfying
\eqref{eq:policy_eligibility}
are considered for selection.
The safety filter then applies
$\vu_{i_{\textup{MI}}}^\star$.
As with MB-CBF, this requires constructing and
solving up to $|\Pi|$ QPs to identify feasible candidates
at each control update, increasing computational
cost with the library size.
Empirically, we observe that minimum-intervention
selection can favor more restrictive constraints,
which may limit the available control authority
and increase the likelihood of infeasibility
as conditions change.

Across all case studies, the compared methods are
evaluated under the same experimental conditions
and nominal controllers, differing only in the
safety-filter algorithm and, where applicable,
the prescribed backup or fallback policies.
All simulation experiments were conducted on a MacBook Air
equipped with an Apple M4 CPU.

\textbf{Numerical Evaluation:}
The finite-horizon rollout values are approximated
by numerical integration, and their gradients are
computed using automatic differentiation through
the rollout computation.
For details of numerical approximation, including cubic-spline
interpolation to reduce gradient errors and the use
of error bounds to adjust the constraint margin, see \cite{knoedler_safety_2025}.

We begin with a planar double-integrator example
to examine safe-set coverage and then consider
increasingly challenging scenarios involving
abrupt changes in dynamics, newly perceived
and moving obstacles, and nonlinear 3D
robot dynamics.

\subsection{Safe Set Under Multiple Fallback Policies}

\begin{figure*}[t]
    \centering
    \includegraphics[width=0.95\textwidth]{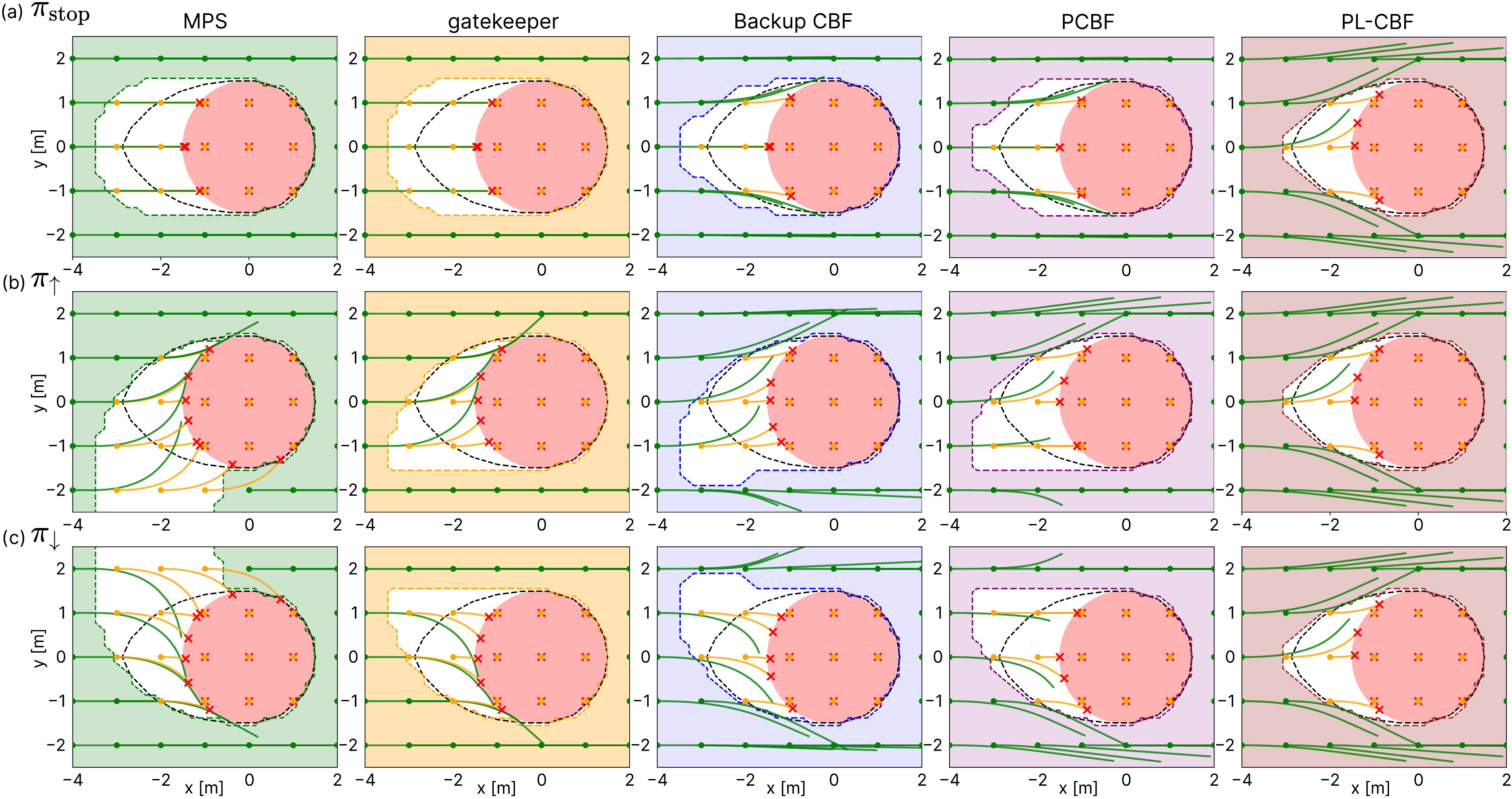}
    \caption{Recovered safe sets for the double-integrator example on the slice $(v_x,v_y)=(2,0)$.
The black dashed line marks the boundary of the viability kernel computed via HJ reachability on
the full four-dimensional state space.
For each initial position on a grid with $1$~m spacing along both axes, we simulate the closed-loop system under each safety filter: green trajectories are safe, while yellow trajectories correspond to unsafe or infeasible rollouts.
    (a) Top row: stopping fallback $\pi_{\textup{stop}}$.
    (b) Middle row: upward evasive fallback $\pi_{\uparrow}$.
    (c) Bottom row: downward evasive fallback $\pi_{\downarrow}$.
PL-CBF uses the same library in all three rows, $\Pi=\{\pi_{\textup{nom}},\pi_{\textup{stop}},\pi_{\uparrow},\pi_{\downarrow}\}$.}
    \label{fig:di_safeset}
\end{figure*}

We first compare safe-set coverage in a planar double-integrator example where a reference viability kernel can be computed. Let $\vx=[\vp^\top\ \vv^\top]^\top\in\RealSpace^4$ with $\vp=[x\ y]^\top$ and $\vv=[v_x\ v_y]^\top$, and dynamics
\begin{equation}\label{eq:di_dyn}
\dot{\vp}=\vv,\quad \dot{\vv}=\vu,\quad \vu\in[-a_{\max},a_{\max}]^2,
\end{equation}
with $a_{\max}=0.5~\textup{m/s}^2$.
The obstacle region is $\calO=\{\vx\mid \|\vp\|_2\le r_{\mathrm{obs}}\}$, where $r_{\mathrm{obs}}>0$ is the obstacle radius, and we use $h(\vx)=\|\vp\|_2-r_{\mathrm{obs}}$, yielding $\calC=\{\vx\mid h(\vx)\ge 0\}$. Since the state dimension is $4$, we can compute the viability kernel for~\eqref{eq:di_dyn} using HJ reachability~\cite{bansal_hamiltonjacobi_2017} and visualize it on the velocity slice $(v_x,v_y)=(2,0)$ in m/s.

To highlight the dependence of single-policy certificates
on the chosen maneuver, we consider three simple fallbacks:
a stopping policy $\pi_{\textup{stop}}$ implemented as
component-wise saturated full deceleration until rest,
i.e., $-a_{\max}\operatorname{sign}(\vv)$, and two evasive
policies $\pi_{\uparrow}$ and $\pi_{\downarrow}$ that apply
saturated acceleration in $+y$ and $-y$, respectively.
Figure~\ref{fig:di_safeset} compares the safe regions recovered under
these three fallback choices.
PL-CBF uses the same library in all three rows,
$\Pi=\{\pi_{\textup{nom}},\pi_{\textup{stop}},\pi_{\uparrow},\pi_{\downarrow}\}$.

Importantly, a single-policy method may recover a different safe set if a different fallback is chosen, but each fixed fallback certifies only the states compatible with that particular maneuver. As a result, replacing the stopping fallback with a different maneuver may recover states that require that specific behavior while losing those that are safe only under stopping.
PL-CBF avoids this tradeoff by evaluating multiple
fallback policies online.
Its finite-horizon safe set is the union of those
associated with individual policies, rather than
being limited to a single fallback maneuver.

\subsection{Indoor Navigation with Emergency Corridor Clearance}
\label{sec:hospital_results}

\begin{figure*}[t]
    \centering
    \includegraphics[width=0.99\textwidth]
        {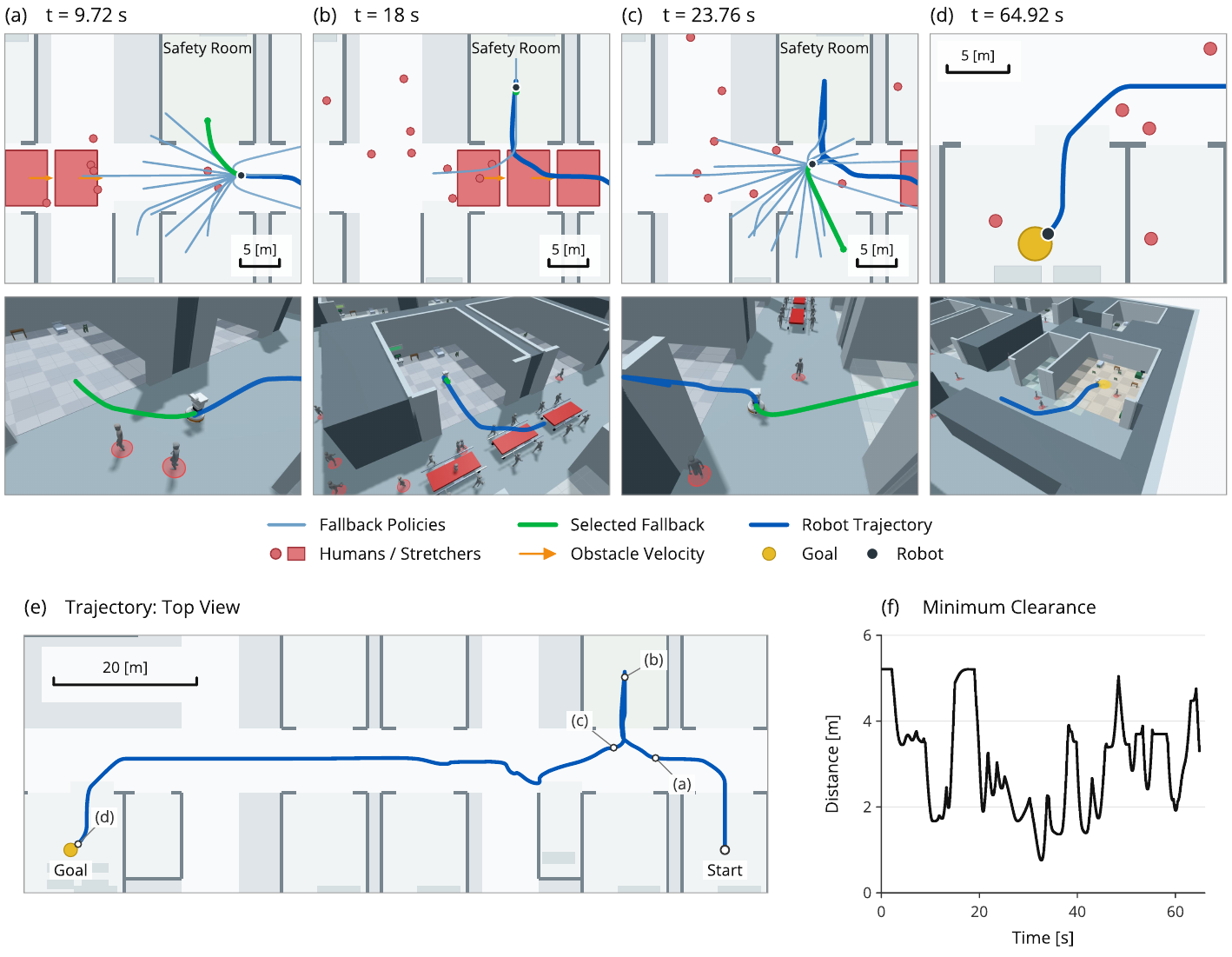}
    \caption{
Indoor navigation with room-entry yielding.
The robot enters a room to clear the corridor, waits
while the stretchers pass, and then resumes navigation
to its original goal.
The blue curve shows the executed path, the green curve
the selected fallback policy rollout, and light-blue curves the other
candidate policy rollouts.
Circles and red rectangles denote humans and stretchers,
respectively.
The selected feasible policy supplies the rollout-based
safety constraint; the QP determines the applied input.
Snapshots (a)--(d) show successive stages of the maneuver,
(e) shows the executed trajectory from above, and
(f) reports the minimum obstacle clearance.
}
    \label{fig:hospital_refuge}
\end{figure*}

We next consider an indoor mobile robot navigation problem in a hospital floor plan
with narrow corridors and adjoining rooms.
This setting is motivated by hospital delivery systems such as
Toyota's Potaro, where a robot stopped in a narrow corridor can
obstruct the passage of stretchers and transport
carts~\cite{yamaoka_achieving_2025}.
Related studies of emergency evacuation also identify nearby
refuge spaces as opportunities for robots to yield without
blocking human movement~\cite{zhou_collision_2026}.
We consider situations in which the robot must clear the
corridor for emergency transport by temporarily entering an
available room while avoiding collisions with moving humans
and static obstacles such as walls.
The robot follows the double-integrator dynamics
in~\eqref{eq:di_dyn}, with an acceleration bound of
$a_{\max}=1.7~\textup{m/s}^2$ and a speed limit of
$2.85~\textup{m/s}$, following an internal mobile robot
specification.

\textbf{Scenario:}
The nominal controller tracks A* waypoints on the known floor
plan using saturated velocity feedback.
Moving obstacles are detected within a $24$~m sensing range,
subject to line of sight.
We evaluate five scenarios with $20$ different configurations
of the human obstacles in each scenario, giving $100$ test
cases in total.
Each case contains $50$ moving human obstacles, modeled as
disks, and two or three stretchers, modeled as rectangles,
moving at $3.0~\textup{m/s}$ along the narrow corridor where
the robot is located.
The stretchers occupy the usable corridor width, preventing
the robot from passing alongside them.
Stopping leaves the robot in their path, while their higher
speed prevents the robot from avoiding the encounter by
retreating indefinitely along the corridor.
Rooms along the corridor provide spaces where the robot can
yield until the stretchers have passed and then continue
toward its goal.

For methods using multiple policies, the library combines
the nominal controller with pre-designed evasive policies
and room-entry policies generated online, both implemented
as proportional-derivative (PD) feedback controllers.
The evasive policies include stopping, reversing, and
$12$ directional feedback policies whose desired headings
span the $180^\circ$ range centered on the nominal-route
direction.
Room-entry policies are constructed online for nearby rooms
whose doorways are reachable and within the robot's sensing
range.
Each room-entry policy guides the robot along a sequence of
waypoints through the doorway into the room, then maintains
the final interior waypoint as its position target for the
remainder of the rollout.
All candidate policies are evaluated over a $7.2$~s rollout horizon.
The single-policy methods use a fixed retrace-waypoint
fallback that follows the preceding A* waypoints in
reverse order, providing a retreat maneuver along the
nominal route.

A trial is successful if the robot reaches the goal region
within $180$~s without physical collision or controller infeasibility.
The success rate is the fraction of the $100$ test cases
that satisfy this criterion.

\begin{table}[t]
\centering
\footnotesize
\setlength{\tabcolsep}{2.5pt}
\renewcommand{\arraystretch}{1.15}
\caption{
Indoor navigation over $100$ test cases.
Success, collision/infeasibility, and timeout rates are reported
as percentages.
Time denotes the average computation time per control step.
}
\label{tab:hospital}

\begin{tabular}{@{}l c c c r@{}}
\hline
\textbf{Algorithm}
& \textbf{Success}
& \makecell{\textbf{Collision/}\\
\textbf{infeasibility}}
& \textbf{Timeout}
& \makecell{\textbf{Time}\\\textbf{[ms]}} \\
\hline
MPS\tablecite{bastani_safe_2021}
& 4\% & 90\% & 6\% & \textbf{2.844} \\
\texttt{gatekeeper}\tablecite{agrawal_gatekeeper_2024}
& 14\% & 86\% & 0\% & 18.018 \\
Backup CBF\tablecite{chen_backup_2021}
& 0\% & 98\% & 2\% & 25.176 \\
PCBF\tablecite{knoedler_safety_2025}
& 0\% & 100\% & 0\% & 3.970 \\
MB-CBF\tablecite{chen_guaranteed_2021}
& 0\% & 100\% & 0\% & 100.440 \\
PL-CBF-MI
& 1\% & 96\% & 3\% & 8.255 \\
\rowcolor{gray!15}
\textbf{PL-CBF}
& \textbf{86\%} & \textbf{12\%} & 2\% & 8.007 \\
\hline
\end{tabular}

\end{table}

\textbf{Results:}
Table~\ref{tab:hospital} reports an $86\%$ success
rate for PL-CBF, with a combined collision/infeasibility rate of $12\%$ and a timeout rate of $2\%$.
Among the single-policy baselines, \texttt{gatekeeper}
achieves the highest success rate of $14\%$,
while MB-CBF and PL-CBF-MI achieve $0\%$
and $1\%$, respectively.
PL-CBF has an average computation time of
$8.007$~ms per control update.

The single-policy methods fail primarily when the
approaching stretchers block the entire corridor.
Although their prescribed backup policies can
avoid individual moving obstacles, they cannot
always accommodate the stretchers, which travel
faster than the robot's maximum speed.
In particular, stopping leaves the robot in the
stretchers' path, while retracing the nominal route
does not provide a viable long-term escape.
Figure~\ref{fig:hospital_refuge} illustrates how
PL-CBF addresses this limitation by incorporating
room-entry policies generated online.
The robot temporarily enters a nearby room to
clear the corridor, waits for the stretchers to
pass, and then resumes navigation toward its
original goal.
This demonstrates the benefit of evaluating
environment-dependent fallback maneuvers at
runtime, particularly when newly detected
safety constraints are not adequately addressed
by the predesigned fallback policies.

The two multiple-policy baselines, MB-CBF and
PL-CBF-MI, also achieve low success rates despite
having access to multiple fallback maneuvers.
Both methods select among feasible candidate
QP solutions by minimizing the immediate
deviation from the nominal input.
This selection can favor continuing along the
nominal route until the approaching stretchers
leave insufficient time or space for a safe
evasive maneuver.
In contrast, PL-CBF selects the constraint
with the largest admissible-input volume,
prioritizing a larger set of available control
actions rather than the smallest immediate
intervention.
The results suggest that this selection criterion
can facilitate earlier and more effective
evasive responses in scenarios where delaying
a maneuver can compromise safety.

Importantly, PL-CBF does not directly execute
the selected fallback policy.
Instead, its rollout determines the safety constraint
supplied to the QP.
Thus, the volume-based criterion selects the constraint,
while the QP computes the minimally intervening control
input subject to that constraint and the input bounds.
Whenever the nominal input satisfies these constraints,
the QP returns it unchanged.
This allows PL-CBF to incorporate room-entry maneuvers generated online
for safety while preserving
the nominal waypoint-following behavior whenever
it remains admissible.

\subsection{Highway Driving with Sudden Friction Change}

\begin{figure*}[t]
    \centering
    \includegraphics[width=0.99\textwidth]{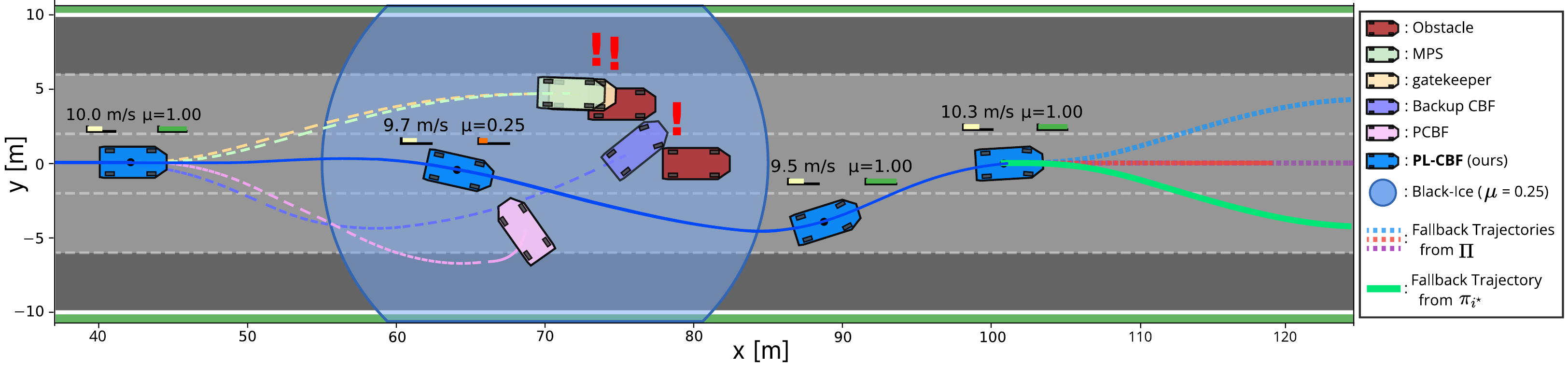}
    \caption{Representative trajectory from the highway driving experiment in \autoref{tab:highway} with MPCC reference speed $v_{\textup{ref}}=10$ m/s. The four single-fallback baselines are shown with the left lane-change controller $\pi_{\textup{left}}$: MPS, \texttt{gatekeeper}, and Backup CBF collide with the obstacle, while PCBF undergoes an oversteer-induced spin after entering the black-ice patch ($\mu$ = 0.25). PL-CBF safely traverses the low-friction segment by selecting from $\Pi=\{\pi_{\textup{stop}},\pi_{\textup{left}},\pi_{\textup{right}},\pi_{\textup{nom}}\}$. Near $x\approx100$ m, dashed lines show candidate rollouts from $\Pi$; the purple dashed curve is the MPCC nominal rollout, and the green curve is the selected rollout from $\pi_{i^\star}$.}
    \label{fig:drift_traj}
\end{figure*}

We evaluate a highway driving scenario in which the road
friction coefficient $\mu$ is revealed online.
We adopt a realistic vehicle model with a Fiala brush tire model~\cite{goh_stable_2024}, which has been experimentally demonstrated to accurately model vehicle behavior even in extreme situations such as racing and drifting~\cite{dallas_control_2025}. Including the global coordinates,
the model has eight states and two inputs:
\begin{equation}
\vx = [\ p_x \  p_y \ \psi \ r \ \beta \
V \ \delta \ \tau \ ]^\top,\qquad
\vu=[\ \dot{\delta} \ \dot{\tau} \ ]^\top,
\end{equation}
where $(p_x,p_y)$ is the planar position of the center of mass~(CoM), $\psi$ is yaw, $r$ is yaw rate, $\beta$ is sideslip, $V$ is speed, $\delta$ is steering angle, and $\tau$ is rear-wheel torque.
The subscripts $\mathrm f$ and $\mathrm r$ denote
the front and rear axles, respectively, and $g$ is
the gravitational acceleration.
 The kinematics and dynamics are
{\footnotesize
\medmuskip=3mu\relax
\begin{align}
\dot{p}_x &= V\cos(\psi+\beta), \quad
\dot{p}_y = V\sin(\psi+\beta), \quad
\dot{\psi} = r, \label{eq:bicycle_kin}\\
\dot{r} &= \frac{a ( F_{x,{\rm f}}\sin\delta + F_{y,{\rm f}}\cos\delta ) - b F_{y,{\rm r}}}{I_z}, \nonumber \\
\dot{\beta} &= \frac{F_{x, {\rm f}}\sin(\delta - \beta) + F_{y, {\rm f}}\cos(\delta - \beta) -F_{x, {\rm r}}\sin\beta + F_{y, {\rm r}}\cos\beta}{mV}-r, \nonumber\\
\dot{V} &= \frac{F_{x, {\rm f}}\cos(\delta - \beta) - F_{y, {\rm f}}\sin(\delta - \beta) + F_{x, {\rm r}}\cos\beta + F_{y, {\rm r}}\sin\beta}{m}, \nonumber
\end{align}
}
Slip angles are $\alpha_{\rm f} = \arctan\left(\frac{V \sin\beta+ar}{V \cos\beta}\right) -\delta ,\,
\alpha_{\rm r} = \arctan\left(\frac{V\sin\beta-br}{V\cos\beta}\right)$,
where $a$ and $b$ are the distances from the CoM to the front and rear axles, respectively, $m$ is the mass, and $I_z$ is the yaw moment of inertia about the CoM. Lateral forces are given by the Fiala brush model, where $C_{\rm c}$ denotes the cornering stiffness:
{\footnotesize
\begin{equation}\label{eq:fiala}
F_y=
\begin{cases}
\begin{aligned}
&-C_{\rm c}\tan\alpha
+\dfrac{C_{\rm c}^2}{3F_{y,\max}}|\tan\alpha|\tan\alpha\\
&\qquad-\dfrac{C_{\rm c}^3}{27F_{y,\max}^2}\tan^3\alpha,
\end{aligned}
& |\alpha|<\alpha_\mathrm{sl},\\[3pt]
-F_{y,\max}\operatorname{sgn}(\alpha),
& |\alpha|\geq\alpha_\mathrm{sl}.
\end{cases}
\end{equation}}
with $\alpha_\mathrm{sl}=\arctan\big(\frac{3F_{y,\max}}{C_{\rm c}}\big)$
and $F_{y,\max}=\sqrt{(\mu F_z)^2-\gamma F_x^2}$.
Here, $F_z$ and $F_x$ are the normal and longitudinal
tire forces, respectively, $\mu$ is the tire--road
friction coefficient, and $\gamma$ is a dimensionless
coefficient weighting the longitudinal-force contribution
to the lateral-force limit. Normal loads follow the static distribution $F_{z,{\rm f}}=\frac{mgb}{a+b}$ and $F_{z,{\rm r}}=\frac{mga}{a+b}$. To model longitudinal slip and friction-limited braking, we saturate the rear longitudinal force using a smooth friction-limited model:
\begin{equation}\label{eq:Fx_mu}
F_{x,{\rm f}}=0,\qquad
F_{x,{\rm r}} = (\mu F_{z,{\rm r}})\tanh \left(\frac{\tau}{r_{\rm w}(\mu F_{z,{\rm r}})}\right),
\end{equation}
where $r_{\rm w}$ is the wheel radius.
Therefore, both longitudinal and lateral forces depend nonlinearly on the road friction $\mu$. In this experiment, the friction coefficient changes across road segments and both the controller and the safety filters use the currently available friction estimate online.

\textbf{Scenario:} The nominal controller is a Model Predictive Contouring Controller (MPCC)~\cite{lam_model_2010} that tracks the centerline of the current lane at target speed $v_{\textup{ref}}$. We omit obstacle costs and collision-avoidance constraints
from the MPCC so that the nominal input can be unsafe,
allowing us to evaluate obstacle avoidance provided
by the safety filter.

The MPCC is updated every $0.05$~s and uses a $30$-step
prediction horizon ($1.5$~s).
Its stage and terminal costs penalize contouring, lag,
heading, and speed-tracking errors, together with yaw rate,
using weights $(30,1,20,50,80)$, respectively.
The MPCC augments the vehicle model with a path-progress
state and an auxiliary progress-rate input.
Successive changes in the steering-rate, torque-rate,
and progress-rate inputs are penalized with weights
$(300,0.5,0.1)$.
The constraints impose
$0\le V\le20~\textup{m/s}$,
$|r|\le2~\textup{rad/s}$,
$|\beta|\le45^\circ$,
$|\delta|\le20^\circ$,
$|\tau|\le4000~\textup{N\,m}$,
$|\dot{\delta}|\le50^\circ/\textup{s}$, and
$|\dot{\tau}|\le8000~\textup{N\,m/s}$.
The optimization is implemented in CasADi and
solved using IPOPT with warm-start initialization,
at most $100$ iterations, and an acceptable
optimality tolerance of $10^{-4}$.

Obstacles are detected only within a limited sensing range, and the friction estimate $\mu$ is updated only after the vehicle enters the corresponding road segment.
We place a black-ice patch with $\mu=0.25$ near the obstacles, while the remainder of the road has $\mu=1.0$. We evaluate $100$ trials, each with one or two randomly
sampled obstacles within the low-friction patch, constrained
to occupy distinct lanes of the three-lane roadway. The policy library is $\Pi=\{\pi_{\textup{stop}},\pi_{\textup{left}},\pi_{\textup{right}},\pi_{\textup{nom}}\}$, where $\pi_{\textup{stop}}$ applies maximum deceleration, $\pi_{\textup{left}}$ and $\pi_{\textup{right}}$ are tuned feedback lane-change controllers to the target lane, and $\pi_{\textup{nom}}$ is the MPCC. The fallback policies are evaluated over a
$3.0$~s rollout horizon. Under nominal friction and sufficient sensing, $\pi_{\textup{stop}}$ is a valid backup; however, when an obstacle lies on or beyond the low-friction segment, braking alone may no longer be valid under the input constraints and can lead to collision because the reduced friction envelope in \eqref{eq:fiala}--\eqref{eq:Fx_mu} limits the available tire forces.

\begin{table}[t]
\centering
\footnotesize
\setlength{\tabcolsep}{2.5pt}
\renewcommand{\arraystretch}{1.15}
\caption{
Highway driving with sudden friction change.
Failures include collision or infeasibility over $100$
trials at $v_{\textup{ref}}=10$~m/s and are reported
as percentages.
$v_{\textup{ref}}^{\max}$ is the highest tested speed
passing the separate evaluation over six fixed obstacle
layouts.
The reference-speed searches are capped at $10$~m/s.
``--'' denotes no passing speed found.
Time denotes the average per-step safety-filter
computation time.
}
\label{tab:highway}

\begin{tabular}{@{}l c | c | c | r@{}}
\hline
\textbf{Algorithm}
& \textbf{Fallback}
& \textbf{Fail@10 m/s}
& $v_{\textup{ref}}^{\max}$ \textbf{[m/s]}
& \textbf{Time [ms]} \\
\hline
\multirow{3}{*}{MPS\tablecite{bastani_safe_2021}}
& $\pi_{\textup{stop}}$
& 78.0\% & 2.50
& \multirow{3}{*}{\textbf{0.203}} \\
& $\pi_{\textup{left}}$
& 73.0\% & -- & \\
& $\pi_{\textup{right}}$
& 80.0\% & -- & \\
\hline
\multirow{3}{*}{
\texttt{gatekeeper}\tablecite{agrawal_gatekeeper_2024}}
& $\pi_{\textup{stop}}$
& 78.0\% & 2.75
& \multirow{3}{*}{1.447} \\
& $\pi_{\textup{left}}$
& 76.0\% & -- & \\
& $\pi_{\textup{right}}$
& 82.0\% & -- & \\
\hline
\multirow{3}{*}{Backup CBF\tablecite{chen_backup_2021}}
& $\pi_{\textup{stop}}$
& 62.0\% & 2.75
& \multirow{3}{*}{23.179} \\
& $\pi_{\textup{left}}$
& 34.0\% & 5.00 & \\
& $\pi_{\textup{right}}$
& 35.0\% & 5.00 & \\
\hline
\multirow{3}{*}{PCBF\tablecite{knoedler_safety_2025}}
& $\pi_{\textup{stop}}$
& 78.0\% & 2.75
& \multirow{3}{*}{3.988} \\
& $\pi_{\textup{left}}$
& 30.0\% & 4.50 & \\
& $\pi_{\textup{right}}$
& 14.0\% & 5.00 & \\
\hline
MB-CBF\tablecite{chen_guaranteed_2021}
& $\Pi$
& \textbf{0.0\%}
& \textbf{10.00}
& 493.509 \\
\hline
PL-CBF-MI
& $\Pi$
& 34.0\%
& 7.75
& 19.368 \\
\hline
\rowcolor{gray!15}
\textbf{PL-CBF}
& $\Pi$
& \textbf{0.0\%}
& \textbf{10.00}
& 5.949 \\
\hline
\end{tabular}
\end{table}

\textbf{Results:}
At $v_{\textup{ref}}=10$~m/s, PL-CBF and MB-CBF
both achieve a $0\%$ failure rate, compared with
$34\%$ for PL-CBF-MI and at least $14\%$ for every
single-fallback configuration
(Table~\ref{tab:highway}).
For Backup CBF and PCBF, lane-change fallbacks reduce
failures relative to braking alone, but neither fixed
lane-change direction avoids failure across all trials.
These results illustrate the value of retaining lateral
evasive options when reduced friction limits braking,
while adapting the maneuver to the occupied lanes.
An illustrative trial is shown in
Fig.~\ref{fig:drift_traj}.

In the separate reference-speed evaluation, PL-CBF
and MB-CBF pass at the tested ceiling of $10$~m/s,
while PL-CBF-MI reaches $7.75$~m/s.
PL-CBF has an average computation time of
$5.949$~ms per control step, compared with
$19.368$~ms for PL-CBF-MI and $493.509$~ms
for MB-CBF.
Although both PL-CBF and MB-CBF record zero
failures at $10$~m/s in their respective trials,
PL-CBF is approximately $83$ times faster.
The higher computational cost of MB-CBF arises from
evaluating trajectory-dependent constraints and
solving a separate QP for each backup policy.
This highlights the computational advantage
of solving a single QP rather than
a separate Backup-CBF QP for each policy
in the library.

\subsection{3D Quadrotor Navigation with Linearized Dynamics}

\begin{figure*}[t]
    \centering
    \includegraphics[width=0.99\textwidth]{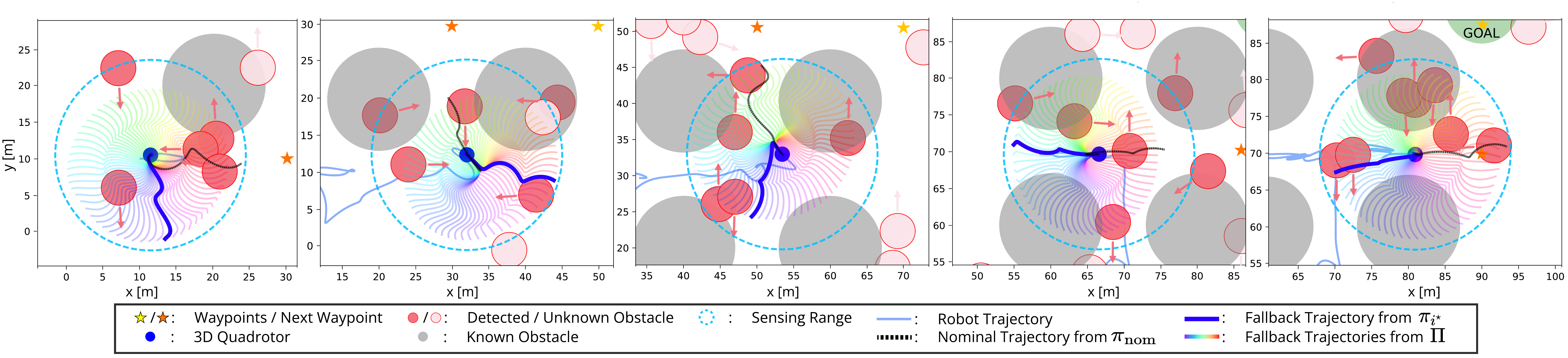}
    \caption{Warehouse navigation using PL-CBF with $P=64$. The robot tracks sequential waypoints (stars) while reacting to newly sensed dynamic obstacles.
At each step,
PL-CBF selects a feasible constraint by admissible-input volume and applies the QP-based safety filter online (\autoref{alg:plcbf}); the thick blue line shows the corresponding selected fallback rollout.}
    \label{fig:warehouse_policy_selection}
\end{figure*}

To evaluate safety filtering for higher-dimensional dynamics,
we consider quadrotor navigation in a crowded warehouse. The robot is tasked with visiting waypoints sequentially while avoiding both known static obstacles (gray) and newly detected dynamic obstacles (red).
We use a 3D quadrotor model linearized about level hover at zero yaw, with 12 states and 4 inputs~\cite{annaswamy_integration_2023}:
\begin{equation*}
\vx = [\,x\ y\ z\ \dot{x}\ \dot{y}\ \dot{z}\ \phi\ \theta\ \psi\ \dot{\phi}\ \dot{\theta}\ \dot{\psi}\,]^\top, \
\vu = [\,\tau_{\phi}\ \tau_{\theta}\ \tau_{\psi}\ f_z\,]^\top,
\end{equation*}
with equations of motion
\begin{equation}
\begin{aligned}
\ddot{x} &= g\theta,
& \ddot{y} &= -g\phi,
& \ddot{z} &= \frac{f_z-mg}{m}, \\[3pt]
\ddot{\theta} &= \frac{\tau_{\theta}}{I_y},
& \ddot{\phi} &= \frac{\tau_{\phi}}{I_x},
& \ddot{\psi} &= \frac{1}{I_z}\tau_{\psi}.
\end{aligned}
\end{equation}
Here, $(x,y,z)$ is the position, $(\phi,\theta,\psi)$
are the roll, pitch, and yaw angles, respectively,
$f_z$ is the total thrust, and
$\tau_\phi$, $\tau_\theta$, and $\tau_\psi$ are the
corresponding body torques. The mass is $m$, and
$I_x$, $I_y$, and $I_z$ are the moments of inertia
about the body-fixed axes.
We impose realistic physical limits, including a thrust-to-weight ratio of $2.36$.
The model has 12 states, and the horizontal position coordinates have relative degree four with respect to the corresponding attitude-control inputs. The vertical position has relative degree two with respect to thrust. These high-order dynamics and the input bounds complicate direct geometric CBF synthesis.

\textbf{Scenario:}
The environment contains $16$ known static obstacles and $45$ dynamic obstacles that are initially unobserved and become available only within the sensing range. We evaluate $100$ randomized trials by sampling the dynamic obstacles' initial positions and velocities.
The nominal controller tracks sequential waypoints,
and all single-policy baselines use the same retrace
fallback toward the previous waypoint.
The policy library contains $P$ directional
feedback policies with headings uniformly spaced over
$[0,2\pi)$, together with stopping and nominal
waypoint-tracking policies, giving $|\Pi|=P+2$.
We vary $P\in\{4,8,16,32,64\}$ for PL-CBF and
evaluate MB-CBF and PL-CBF-MI at $P=64$.
The fallback policies are evaluated over a
$4$~s rollout horizon.

\textbf{MI-MPC baseline:}
We also compare with a mixed-integer MPC~(MI-MPC) baseline that jointly
optimizes continuous state and input sequences and
a binary choice among directional maneuver rollouts.
Predicted obstacle clearance is evaluated along the
candidate rollouts, while big-$M$ constraints bound
deviations from the selected rollout under the
prediction dynamics and input bounds. We use the same maneuver set as PL-CBF with $P=64$ and follow a standard MI-MPC formulation as described in~\cite{borrelli_predictive_2017}.

\begin{table}[t]
\centering
\footnotesize
\setlength{\tabcolsep}{3pt}
\renewcommand{\arraystretch}{1.15}
\caption{
Warehouse navigation over $100$ trials.
Failures combine collision and infeasibility and are
reported as counts (percentages).
Single-policy baselines use the retrace fallback.
PL-CBF, MB-CBF, and PL-CBF-MI use $P$ directional
policies together with stopping and nominal
waypoint-tracking policies ($|\Pi|=P+2$).
Time denotes the average computation time per
control step.
}
\label{tab:warehouse}
\begin{tabular}{@{}l c c r@{}}
\hline
\textbf{Algorithm}
& $P$
& \makecell{\textbf{Collision/}\\
\textbf{infeasibility}}
& \makecell{\textbf{Time}\\\textbf{[ms]}} \\
\hline
MPS\tablecite{bastani_safe_2021}
& -- & 77/100 (77.0\%) & \textbf{0.176} \\
\texttt{gatekeeper}\tablecite{agrawal_gatekeeper_2024}
& -- & 68/100 (68.0\%) & 2.841 \\
Backup CBF\tablecite{chen_backup_2021}
& -- & 74/100 (74.0\%) & 97.398 \\
PCBF\tablecite{knoedler_safety_2025}
& -- & 75/100 (75.0\%) & 11.316 \\
MI-MPC\tablecite{borrelli_predictive_2017}
& 64 & 73/100 (73.0\%) & 688.230 \\
MB-CBF\tablecite{chen_guaranteed_2021}
& 64 & 85/100 (85.0\%) & 140.861 \\
PL-CBF-MI
& 64 & 82/100 (82.0\%) & 832.870 \\
\hline
\rowcolor{gray!15}
PL-CBF
& 4 & 8/100 (8.0\%) & 13.485 \\
\rowcolor{gray!15}
PL-CBF
& 8 & 6/100 (6.0\%) & 14.199 \\
\rowcolor{gray!15}
PL-CBF
& 16 & 5/100 (5.0\%) & 15.512 \\
\rowcolor{gray!15}
PL-CBF
& 32 & 3/100 (3.0\%) & 18.292 \\
\rowcolor{gray!15}
\textbf{PL-CBF}
& 64 & \textbf{0/100 (0.0\%)} & 23.375 \\
\hline
\end{tabular}
\end{table}

\textbf{Results.}
Table~\ref{tab:warehouse} shows that PL-CBF records
an $8\%$ failure rate with only four directional
policies and no failures at $P=64$, whereas the
comparison methods have failure rates of
$68\%$--$85\%$.
Increasing $P$ from $4$ to $64$ reduces the spacing
between candidate headings from $90^\circ$ to
$5.625^\circ$.
The accompanying reduction in failures suggests
that finer directional coverage helps the safety filter
identify evasive trajectories as moving obstacles
restrict the available paths.

Although the library size increases from $6$ to $66$
policies ($P=4$ to $64$), a factor of $11$,
PL-CBF's average computation time increases by only
a factor of $1.7$, from $13.485$ to $23.375$~ms.
At $P=64$, PL-CBF requires approximately $2.1$ times
the computation time of PCBF ($11.316$~ms),
but reduces the failure rate from $75\%$ to $0\%$.
Even with $66$ candidate policies, PL-CBF's average
computation time corresponds to a control rate of approximately $43$~Hz on a laptop computer.
Furthermore, PL-CBF is approximately $6$ times faster
than MB-CBF ($140.861$~ms) and substantially faster
than PL-CBF-MI ($832.870$~ms). The particularly high computation time of PL-CBF-MI
mainly reflects the overhead of repeated QP
construction and feasibility checking across
all $66$ candidate policies.
This computational efficiency benefits from the
independence of the policy rollouts, which enables
parallel evaluation, together with the
single QP safety filter formulation.
Figure~\ref{fig:warehouse_policy_selection} shows
representative waypoint navigation and the selected
fallback rollouts.

\subsection{3D Quadrotor Navigation with Nonlinear Dynamics}
\label{sec:nl_quad3d_results}

\begin{figure*}[t]
    \centering
    \includegraphics[width=0.99\textwidth]{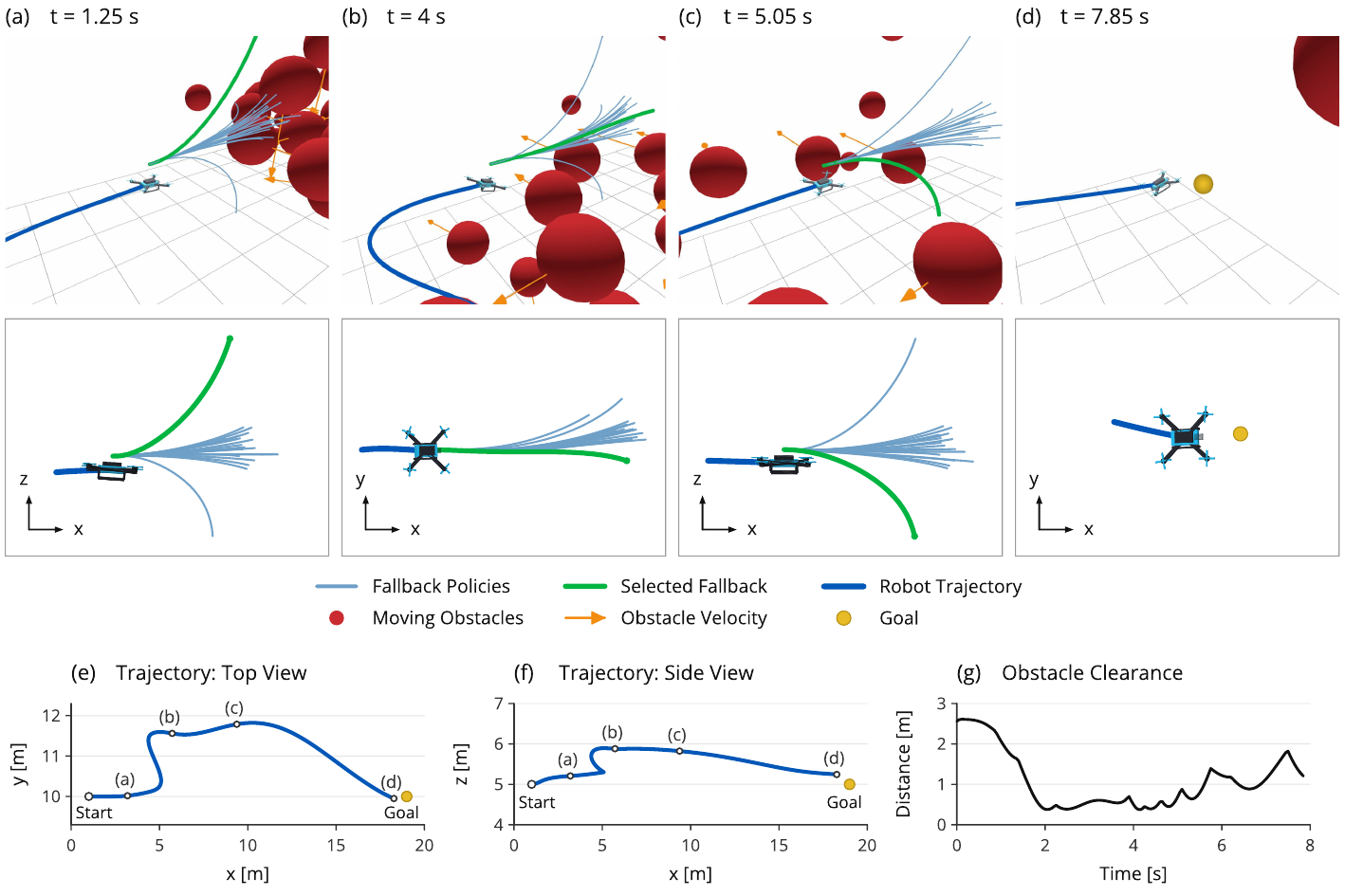}
    \caption{Nonlinear 3D quadrotor navigation among $48$ moving
obstacles. Snapshots (a)--(d) show climbing, lateral, and
descending policy rollouts followed by arrival at the goal.
The blue curve shows the executed trajectory and the thick green curve
the selected policy rollout. Subfigures (e) and (f) show top and side
views of the executed trajectory, respectively, and (g) reports
the obstacle clearance.
This illustration uses $20$ policies and a $1$~s rollout horizon;
Table~\ref{tab:nl_quad3d} reports the benchmark configuration described
in the text.}
    \label{fig:nl_quad3d_navigation}
\end{figure*}

We also demonstrate three-dimensional evasive maneuvers
using a nonlinear quadrotor model.
In Section~\ref{sec:hardware_experiments}, we further
evaluate PL-CBF on a physical quadrotor
in a similar
obstacle-avoidance setting.
The state is
$\vx=[\vp^\top\ \vv^\top\ \boldsymbol{\eta}^\top\
\boldsymbol{\omega}^\top]^\top\in\RealSpace^{12}$,
where $\vp,\vv\in\RealSpace^3$ are the position and velocity in world coordinates, $\boldsymbol{\eta}=[\phi\ \theta\ \psi]^\top$ collects the roll, pitch, and yaw angles, and $\boldsymbol{\omega}$ denotes body angular velocity.
The inputs $\vu=[f_1\ f_2\ f_3\ f_4]^\top$
are bounded, nonnegative rotor thrusts.
The continuous-time dynamics are
\begin{equation}\label{eq:nl_quad_dynamics}
\begin{aligned}
\dot{\vp}&=\vv,\qquad
\dot{\boldsymbol{\eta}}=E(\boldsymbol{\eta})
\boldsymbol{\omega},\\
m\dot{\vv}&=-mg\mathbf{e}_3
+R(\boldsymbol{\eta})\mathbf{e}_3\mathbf{1}^\top\vu,\\
I\dot{\boldsymbol{\omega}}&=B_\tau\vu
-\boldsymbol{\omega}\times I\boldsymbol{\omega},
\end{aligned}
\end{equation}
where $m$ is the mass, $g$ the gravitational acceleration,
$R=R_z(\psi)R_y(\theta)R_x(\phi)$ maps body coordinates
to world coordinates, $E$ maps body rates to
Euler-angle rates, $I$ is the inertia matrix,
$\mathbf{e}_3=[0\ 0\ 1]^\top$,
$\mathbf{1}=[1\ 1\ 1\ 1]^\top$, and
$B_\tau$ is the rotor torque-allocation matrix.

\textbf{Scenario:}
The quadrotor traverses a $20\times20\times10$~m
environment containing $48$ moving spherical obstacles
with radius $0.5$~m and speeds of
$0.75$--$2.25$~m/s.
Obstacles reflect at the environment boundaries.
We evaluate $100$ randomized trials, using the same
obstacle initial conditions and motion for all methods
in each trial.
The robot is represented by a sphere of radius
$0.5$~m centered at $\vp+0.25R\mathbf{e}_3$.
The nominal controller uses cascaded position,
velocity, and attitude feedback to track the goal.

PL-CBF uses $P=12$ velocity-tracking policies with
directions distributed approximately uniformly
over the unit sphere, together with stopping and
nominal policies, giving $|\Pi|=14$.
Each candidate policy is evaluated over
a rollout horizon of $0.75$~s.
Each single-policy baseline uses a retrace
fallback policy that commands the quadrotor toward
previously visited waypoints, providing a maneuver
for retreating along its nominal route.
Obstacles approaching from ahead and behind can make
retracing or hovering unsafe, so the quadrotor may
need to move laterally or change altitude.
The safety filter must reassess these options as obstacles
enter the sensing range and the vehicle's attitude
changes.

\begin{table}[t]
\centering
\footnotesize
\setlength{\tabcolsep}{3pt}
\renewcommand{\arraystretch}{1.15}
\caption{
Nonlinear 3D quadrotor navigation over $100$
randomized trials.
Success and collision/infeasibility rates are reported as percentages.
Time denotes the average computation time per
control step.
}
\label{tab:nl_quad3d}
\begin{tabular}{@{}l c c r@{}}
\hline
\textbf{Algorithm}
& \textbf{Success}
& \makecell{\textbf{Collision/}\\
\textbf{infeasibility}}
& \textbf{Time [ms]} \\
\hline
MPS\tablecite{bastani_safe_2021}
& 0.0\% & 100.0\% & \textbf{0.204} \\
\texttt{gatekeeper}\tablecite{agrawal_gatekeeper_2024}
& 0.0\% & 100.0\% & 3.089 \\
Backup CBF\tablecite{chen_backup_2021}
& 85.0\% & 15.0\% & 4.974 \\
PCBF\tablecite{knoedler_safety_2025}
& 99.0\% & 1.0\% & 0.595 \\
MB-CBF\tablecite{chen_guaranteed_2021}
& 16.0\% & 84.0\% & 12.375 \\
PL-CBF-MI
& 84.0\% & 16.0\% & 3.321 \\
\rowcolor{gray!15}
\textbf{PL-CBF}
& \textbf{100.0\%} & \textbf{0.0\%} & 2.464 \\
\hline
\end{tabular}
\end{table}

\textbf{Results:}
Table~\ref{tab:nl_quad3d} reports a $100\%$ success
rate for PL-CBF over the $100$ randomized trials,
compared with $99\%$ for PCBF, $85\%$ for
Backup CBF, $84\%$ for PL-CBF-MI, and $16\%$
for MB-CBF.
MPS and \texttt{gatekeeper} collide in all trials.

Figure~\ref{fig:nl_quad3d_navigation} illustrates
how climbing and lateral maneuvers provide
alternative routes as the surrounding obstacles move.
The nonlinear rollouts account for the attitude
changes required to redirect thrust, allowing
candidate maneuvers to be evaluated under the
rotor-thrust bounds.

\section{HARDWARE EXPERIMENTS \label{sec:hardware_experiments}}

We evaluate PL-CBF on a holonomic mobile rover and a
custom-built quadrotor.
These experiments examine its integration with existing
chassis and flight controllers during waypoint navigation
among moving obstacles.

\subsection{Planar Navigation with a Holonomic Rover}
\label{sec:hardware_rover}

\begin{figure*}[t]
    \centering
    \includegraphics[width=0.95\textwidth]{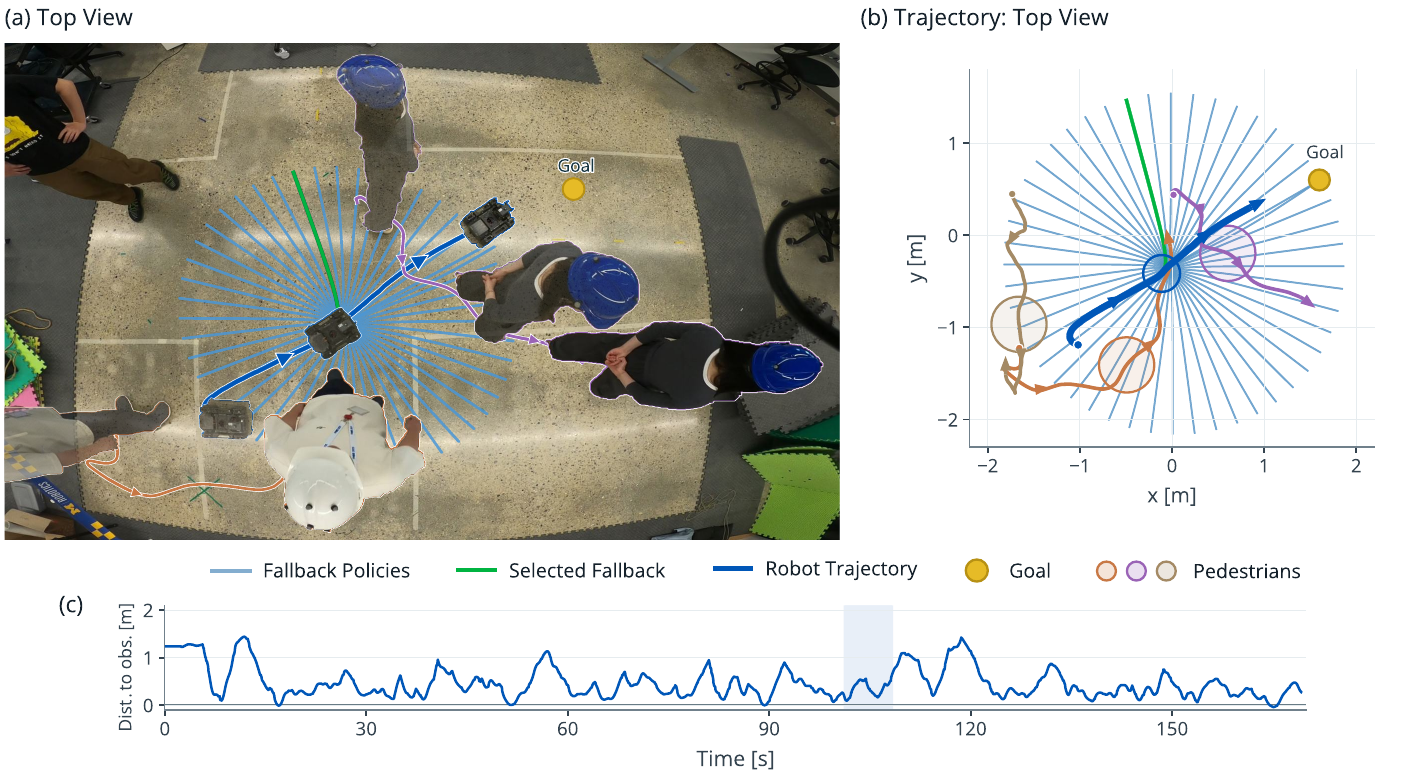}
    \caption{
    Rover hardware experiment with three moving pedestrians.
    (a) Top view composite showing successive rover and
    pedestrian positions, with candidate fallback policy rollouts overlaid.
    (b) Corresponding rover and pedestrian trajectories with
policy rollouts, plotted from recorded ROS bag data.
    Arrows indicate the direction of motion, and the rover
    and pedestrian circles represent their footprints.
    The executed rover trajectory is blue, candidate policy
    rollouts are light blue, the selected fallback policy is green,
    and the current goal is yellow.
    Pedestrian trajectories are shown in orange, purple, and light brown.
    (c) Minimum clearance between the rover and the pedestrians over the recorded experiment.
    The shaded interval marks the time window corresponding
    to the snapshots and trajectories shown in (a) and (b).
    }
    \label{fig:hardware_rover}
\end{figure*}

\textbf{Setup:}
Motivated by the indoor navigation task in
Section~\ref{sec:hospital_results}, we use a RoboMaster EP with an omnidirectional mobile
base to follow planar waypoints while three people move
through its workspace.
The operating region is approximately
$4.0 \times 3.0$~m.
Vicon provides the rover and pedestrian positions,
and filtered finite differences provide velocity estimates.
The implementation uses ROS~2 Humble for communication~\cite{macenski_robot_2022},
JAX in Python for rollout evaluation and differentiation~\cite{jax2018github}, and OSQP
for solving the QP.

The robot is modeled as a double integrator
as in~\eqref{eq:di_dyn}.
The robot's maximum speed is set to $0.8$~m/s.
The nominal controller is a PD controller that tracks
successive waypoints.
Upon reaching each waypoint, the next goal is randomly
selected from the four corners of the workspace.
The robot is tasked with following these waypoints while
avoiding pedestrians for $3$~min.

The policy library contains $48$ directional
velocity-tracking policies with uniformly spaced
planar headings, together with stopping and nominal
waypoint-tracking policies, giving $|\Pi|=50$.
The directional policies track a speed of $0.7$~m/s.
Candidate fallback policies are evaluated over a $3.0$~s
horizon with a $0.05$~s integration step.
The PL-CBF ROS node is configured to publish control
commands at $20$~Hz.
The QP enforces an acceleration input bound of
$a_{\max}=2$~m/s$^2$ along each planar axis.
The rover and pedestrians
are represented by circles of radii $0.20$ and $0.30$~m,
respectively.

\textbf{Results:}
Figure~\ref{fig:hardware_rover}(a)--(b) illustrates
the rover adjusting its motion around moving pedestrians
while progressing toward the current waypoint.
The candidate fallback rollouts provide alternative evasive
directions as the pedestrians move through the workspace.
The selected fallback, shown in green, differs from
the executed trajectory, shown in blue.
This distinction reflects the role of the fallback:
its rollout defines the safety constraint for the QP,
while the nominal controller continues to specify
motion toward the waypoint.

Figure~\ref{fig:hardware_rover}(c) reports the minimum
clearance to the nearest obstacle.
The shaded interval corresponds to the interaction
shown in Figure~\ref{fig:hardware_rover}(a)--(b).
The full experiment is shown in the supplementary video.

\subsection{Quadrotor Flight Among Dynamic Obstacles}
\label{sec:hardware_quadrotor}

\begin{figure*}[t]
    \centering
    \includegraphics[width=\textwidth]{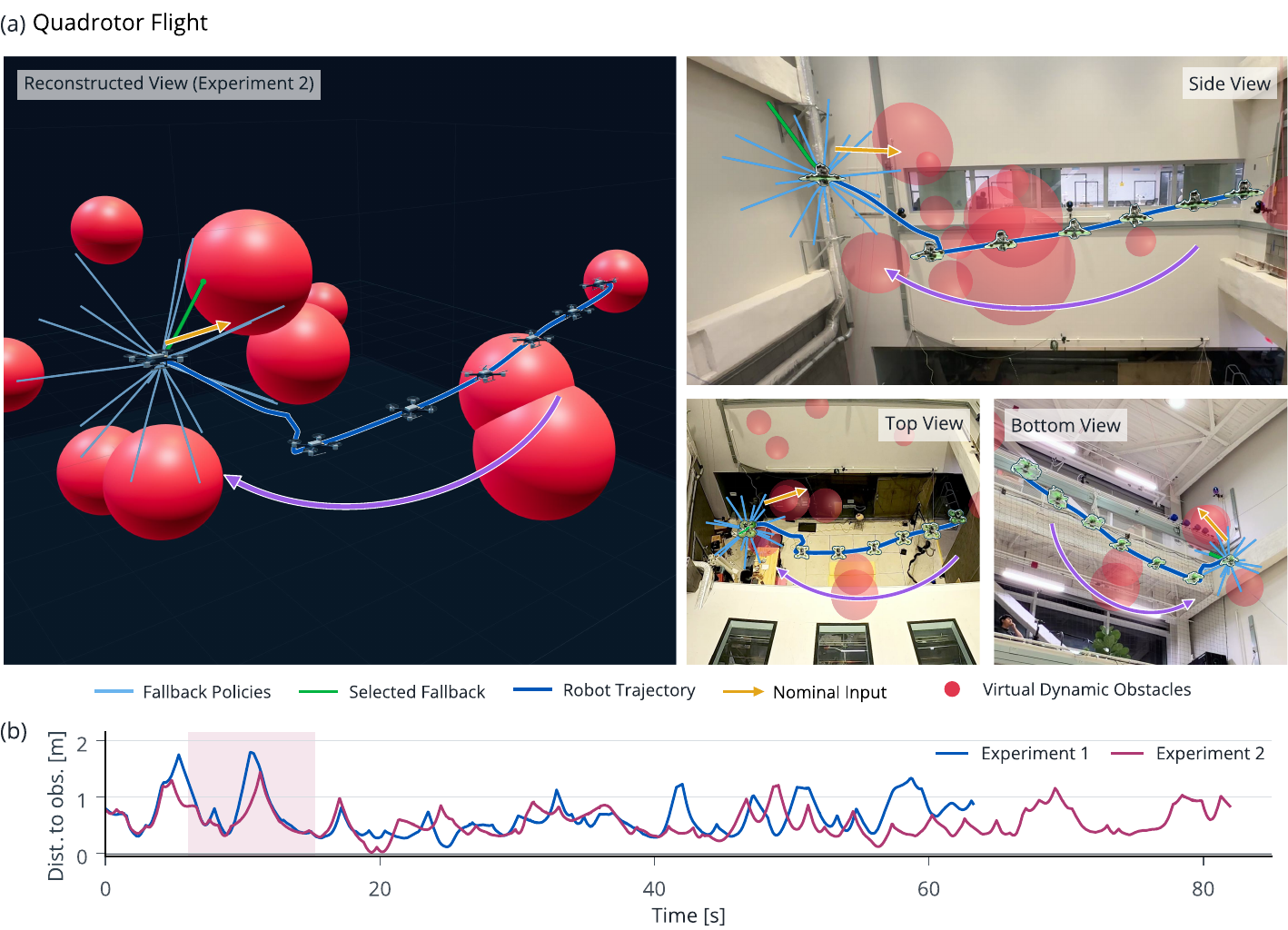}
    \caption{
    Quadrotor hardware experiments with $10$ and $12$
    virtual dynamic obstacles in Experiments~1 and~2,
    respectively.
    (a) Representative flight trajectory from Experiment~2.
    The left panel is a visualization reconstructed
    by postprocessing the recorded ROS bag.
    The camera views on the right were recorded from
    the third floor (top view), second floor (side view),
    and first floor (bottom view).
    The red spheres represent virtual dynamic obstacles
    overlaid on the recorded videos during postprocessing.
    The robot trajectory is blue, candidate policy
    rollouts are light blue, the selected fallback is green,
    and the orange arrow indicates the nominal input direction.
    The obstacle and policy rollout overlays depict
    one instant within the illustrated sequence.
    (b) Minimum clearance to the virtual obstacles
    in Experiment~1 (blue) and Experiment~2 (magenta).
    The shaded interval marks the time window corresponding
    to the Experiment~2 sequence shown in                                                          (a).
    }
    \label{fig:hardware_quadrotor}
\end{figure*}

\textbf{Setup:}
We evaluate three-dimensional waypoint navigation on
a custom-built quadrotor equipped with a Pixhawk flight
controller running PX4 with a custom implementation of
the geometric tracking controller~\cite{lee_geometric_2010}.
An onboard NVIDIA Jetson Orin NX runs the entire algorithm including the nominal
controller and PL-CBF safety filter. The quadrotor dynamics are modeled as
in~\eqref{eq:nl_quad_dynamics}.
The experiments are performed in an indoor flight arena
equipped with a Vicon motion capture system.
PL-CBF obtains position, velocity, acceleration, and
heading from the PX4 local state estimate, which
incorporates position measurements from Vicon.
The remaining quadrotor and actuator states required
for prediction are reconstructed from these signals.

The nominal controller computes a desired velocity toward
the current waypoint and uses the current reference state
to determine a nominal acceleration target
$\vu_{\textup{nom}}$.
PL-CBF filters this target and passes the resulting command
to a reference generator whose position $\vp_r$,
velocity $\vv_r$, and acceleration $\mathbf{a}_r$
evolve according to:
\begin{equation}
\label{eq:hardware_quad_reference}
\begin{aligned}
\dot{\vp}_r &= \vv_r,
\qquad
\dot{\vv}_r = \mathbf{a}_r,\\
\dot{\mathbf{a}}_r
&= \frac{\vu-\mathbf{a}_r}{\tau_r},
\qquad
\tau_r=0.25~\textup{s}.
\end{aligned}
\end{equation}
Here, $\vu=[u_x\ u_y\ u_z]^\top\in\RealSpace^3$
denotes the acceleration target filtered by PL-CBF,
while the rotor thrusts in~\eqref{eq:nl_quad_dynamics}
are computed by the low-level controller through
control allocation. The time constant $\tau_r$
determines how quickly the reference acceleration
responds to this target.
The reference position, velocity, and acceleration are
not reset when $\vu$ changes and therefore remain
continuous across control updates.

The generated position, velocity, acceleration, and
jerk $\dot{\mathbf{a}}_r$ are published through the
ROS~2/PX4 trajectory interface at $20$~Hz.
On the Pixhawk, the geometric tracking controller
combines these references with the estimated
quadrotor state to compute a desired collective
thrust and a body-moment vector~\cite{lee_geometric_2010}.
The PX4 control allocation algorithm then maps these
commands to individual motor commands using the
configured rotor geometry and actuator limits.

The library contains $18$ directional velocity-tracking
policies in three dimensions, together with stopping
and nominal waypoint-tracking policies, giving
$|\Pi|=20$.
Each policy supplies an acceleration target to the
reference generator.
The nominal and directional target speeds are
$0.8$~m/s, and the acceleration targets satisfy
$|u_x|,|u_y|\leq2$~m/s$^2$ and
$|u_z|\leq1$~m/s$^2$,
where $u_x$, $u_y$, and $u_z$ denote the acceleration
targets along the $x$, $y$, and $z$ axes, respectively.

Each candidate rollout predicts the response of the
reference generator, geometric tracking controller,
control allocation, and nonlinear
quadrotor dynamics over a $1.5$~s horizon.
Within the prediction model, the geometric controller
and allocation are updated every $0.01$~s.
The allocation model maps the requested thrust and
moments from the geometric controller to squared rotor speeds using the inverse
allocation matrix and clips them to the actuator bounds.
The resulting rotor-speed commands drive a first-order
motor model with a $0.02$~s time constant.

The quadrotor and motor dynamics are integrated using
RK4 with a $0.01$~s step, while the reference dynamics
are propagated analytically between updates of the
acceleration target $\vu$.
The rollouts and their derivatives are computed using JAX.
The PL-CBF safety filter is implemented as a ROS~2 node.

We report results from two flight experiments with
$10$ and $12$ virtual moving spherical obstacles,
respectively.
The spheres move at $0.45$--$0.90$~m/s and have radii
between $0.23$ and $0.38$~m.
Their positions and velocities are supplied to the safety filter during
flight, and the virtual obstacles are overlaid on the
recorded videos during postprocessing.
The quadrotor is modeled as a sphere of radius $0.15$~m.
The permitted flight region is bounded horizontally
by the laboratory boundaries and vertically by
altitude limits of $5$ and $8$~m.

\textbf{Results:}
Both experiments reach $11$ waypoint goals and finish
in hover.
Figure~\ref{fig:hardware_quadrotor}(a) illustrates
a representative flight sequence from Experiment~2.
The reconstructed and recorded views show lateral
motion and altitude changes during waypoint navigation.
At the displayed policy-evaluation instant, the
selected fallback provides a climbing rollout
that differs from the nominal input direction.
The rollout-based formulation allows PL-CBF to incorporate
low-level controllers, such as the geometric tracking
controller, and actuator dynamics into its forward rollouts,
accounting for their effects on the predicted closed-loop
trajectory under each fallback policy.

Figure~\ref{fig:hardware_quadrotor}(b) reports the minimum clearance between the quadrotor and virtual
obstacles.
The shaded interval identifies the Experiment~2
maneuver shown in Figure~\ref{fig:hardware_quadrotor}(a).
The average PL-CBF computation time per control step,
measured onboard the NVIDIA Jetson Orin NX, is
$31.7$~ms for Experiment~1 and $32.6$~ms for Experiment~2.
Recordings of both experiments are available on the project page and in the supplementary video.

\section{CONCLUSION}

We presented PL-CBF, a runtime safety filter for systems
with evolving perceived safety constraints.
PL-CBF evaluates a library of fallback policies through
parallelizable finite-horizon rollouts, selects a feasible
constraint by admissible-input volume, and minimally modifies
the nominal input through a single QP.
This formulation allows complementary fallback maneuvers
to contribute to safety filtering when a designated backup
becomes unsuitable.
Five simulation studies demonstrate the benefit of evaluating
multiple fallback options under limited sensing, moving
obstacles, and changing dynamics, while maintaining
millisecond-scale average computation times.
Experiments with a mobile robot and a custom-built
quadrotor demonstrate real-time operation of PL-CBF
on robotic hardware.

Online construction of fallback policy libraries that cover
a broad range of maneuvers remains an open challenge.
Future work will investigate systematic methods to synthesize
and adapt these policy libraries online as the system dynamics or environmental constraints change.
A central objective is to generate complementary maneuvers
that expand safety coverage while keeping rollout evaluation
and control optimization computationally tractable.

\ifijrr
  \begin{sagebackmatter}{Author Contributions}
Author contribution information will be provided using the CRediT taxonomy in SAGE Track and will appear under this heading at the proof stage.
\end{sagebackmatter}

\begin{sagebackmatter}{Statements and Declarations}

\sagedeclsubhead{Ethical considerations}
Ethical approval was not required for this study because it did not involve human participants, human data, human tissue, or animals.

\sagedeclsubhead{Consent to participate}
Not applicable.

\sagedeclsubhead{Consent for publication}
Not applicable.

\sagedeclsubhead{Declaration of conflicting interests}
The author(s) declared no potential conflicts of interest with respect to the research, authorship, and/or publication of this article.

\end{sagebackmatter}

\begin{sm}
High-resolution supplementary video is available at: \url{https://youtu.be/TzeRALHHIdw}. Code and additional materials are available at the project page: \url{https://www.taekyung.me/plcbf}.
\end{sm}

  \bibliographystyle{SageH}
\else
  \bibliographystyle{IEEEtran}
\fi
\bibliography{references}

@article{mei_perceive_2026,
	title = {Perceive with confidence: {Statistical} safety assurances for navigation with learning-based perception},
	volume = {45},
	issn = {0278-3649},
	shorttitle = {Perceive with confidence},
	doi = {10.1177/02783649251378151},
	language = {EN},
	number = {6},
	urldate = {2026-10-31},
	journal = {The International Journal of Robotics Research},
	publisher = {SAGE Publications Ltd STM},
	author = {Mei, Zhiting and Dixit, Anushri and Booker, Meghan and Zhou, Emily and Storey-Matsutani, Mariko and Ren, Allen Z. and Shorinwa, Ola and Majumdar, Anirudha},
	year = {2026},
	pages = {938--967},
}

@article{axelrod_provably_2018,
	title = {Provably safe robot navigation with obstacle uncertainty},
	volume = {37},
	issn = {0278-3649},
	doi = {10.1177/0278364918778338},
	language = {EN},
	number = {13-14},
	urldate = {2026-10-31},
	journal = {The International Journal of Robotics Research},
	publisher = {SAGE Publications Ltd STM},
	author = {Axelrod, Brian and Kaelbling, Leslie Pack and Lozano-Pérez, Tomás},
	year = {2018},
	pages = {1760--1774},
}

@article{zhou_collision_2026,
	title = {Beyond {Collision} {Avoidance}: {Multi}-{Robot} {Yielding} and {Spatial} {Affordance} in {Emergency} {Evacuations}},
	doi = {10.48550/arXiv.2605.16115},
	urldate = {2026-10-28},
	journal = {{arXiv} preprint {arXiv}:2605.16115},
	author = {Zhou, Ning and Hunt, Edmund R. and Bode, Nikolai W. F.},
	year = {2026},
}

@article{yamaoka_achieving_2025,
	title = {Achieving {Lean} {Hospital} {Logistics} {Using} {Potaro} (an {Automated} {Delivery} {Robot} {Designed} to {Coexist} with {People}) and a {Multi}-{Robot} {Coordination} {System}},
	volume = {70},
	language = {en},
	number = {2},
	journal = {Toyota Technical Review},
	author = {Yamaoka, Masaaki and Matsui, Takeshi and Oda, Shiro and Toyoshima, Satoshi and Ikeuchi, Yoshito and Taniai, Osamu and Kurahashi, Yukinori},
	year = {2025},
	pages = {20--23},
}

@article{pivtoraiko_differentially_2009,
	title = {Differentially constrained mobile robot motion planning in state lattices},
	volume = {26},
	copyright = {http://onlinelibrary.wiley.com/termsAndConditions\#vor},
	issn = {1556-4959, 1556-4967},
	doi = {10.1002/rob.20285},
	language = {en},
	number = {3},
	urldate = {2026-10-27},
	journal = {Journal of Field Robotics},
	author = {Pivtoraiko, Mihail and Knepper, Ross A. and Kelly, Alonzo},
	year = {2009},
	pages = {308--333},
}

@article{marichal_slices_2008,
	title = {Slices, slabs, and sections of the unit hypercube},
	issn = {19313365},
	doi = {10.61091/ojac-301},
	language = {en},
	number = {3},
	urldate = {2026-10-27},
	journal = {Online Journal of Analytic Combinatorics},
	author = {Marichal, Jean-Luc and Mossinghoff, Michael J.},
	year = {2008},
	pages = {1--11},
}

@article{ong_safespace_2026a,
	title = {{SafeSpace}: {Aggregating} {Safe} {Sets} from {Backup} {Control} {Barrier} {Functions} under {Input} {Constraints}},
	shorttitle = {{SafeSpace}},
	doi = {10.48550/arXiv.2604.03536},
	urldate = {2026-04-23},
	journal = {{arXiv} preprint {arXiv}:2604.03536},
	author = {Ong, Pio and van Wijk, David E. J. and de Sa, Massimiliano and Burdick, Joel W. and Ames, Aaron D.},
	year = {2026},
}

@inproceedings{kim_backupbased_2026,
	title = {Backup-{Based} {Safety} {Filters}: {A} {Comparative} {Review} of {Backup} {CBF}, {Model} {Predictive} {Shielding}, and gatekeeper},
	shorttitle = {Backup-{Based} {Safety} {Filters}},
	doi = {10.48550/arXiv.2604.02401},
	urldate = {2026-04-06},
	booktitle = {{IEEE} {Conference} on {Decision} and {Control} ({CDC})},
	author = {Kim, Taekyung and Menon, Aswin D. and Trivedi, Akshunn and Panagou, Dimitra},
	year = {2026},
}

@inproceedings{agrawal_certifiablycorrect_2025,
	title = {Certifiably-{Correct} {Mapping} for {Safe} {Navigation} {Despite} {Odometry} {Drift}},
	doi = {10.48550/arXiv.2504.18713},
	urldate = {2026-02-11},
	booktitle = {Robotics: {Science} and {Systems} ({RSS})},
	author = {Agrawal, Devansh R. and Kim, Taekyung and Govindjee, Rajiv and Adeshara, Trushant and Yu, Jiangbo and Ravikumar, Anurekha and Panagou, Dimitra},
	year = {2025},
}

@inproceedings{janwani_learningbased_2024,
	title = {A {Learning}-{Based} {Framework} for {Safe} {Human}-{Robot} {Collaboration} with {Multiple} {Backup} {Control} {Barrier} {Functions}},
	doi = {10.48550/arXiv.2310.05865},
	urldate = {2026-06-17},
	booktitle = {{IEEE} {International} {Conference} on {Robotics} and {Automation} ({ICRA})},
	author = {Janwani, Neil C. and Daş, Ersin and Touma, Thomas and Wei, Skylar X. and Molnar, Tamas G. and Burdick, Joel W.},
	year = {2024},
	pages = {11676--11682},
}

@inproceedings{yang_safe_2024,
	title = {Safe {Control} {Synthesis} for {Hybrid} {Systems} through {Local} {Control} {Barrier} {Functions}},
	issn = {2378-5861},
	doi = {10.23919/ACC60939.2024.10644200},
	urldate = {2026-03-04},
	booktitle = {American {Control} {Conference} ({ACC})},
	author = {Yang, Shuo and Black, Mitchell and Fainekos, Georgios and Hoxha, Bardh and Okamoto, Hideki and Mangharam, Rahul},
	year = {2024},
	pages = {344--351},
}

@article{goh_stable_2024,
	title = {Beyond the stable handling limits: nonlinear model predictive control for highly transient autonomous drifting},
	volume = {62},
	issn = {0042-3114},
	shorttitle = {Beyond the stable handling limits},
	doi = {10.1080/00423114.2023.2297799},
	number = {10},
	urldate = {2026-02-24},
	journal = {Vehicle System Dynamics},
	publisher = {Taylor \& Francis},
	author = {Goh, Jonathan Y. M. and Thompson, Michael and Dallas, James and Balachandran, Avinash},
	year = {2024},
	pages = {2590--2613},
}

@article{deng_verification_2016,
	title = {Verification of {Hybrid} {Automata} {Diagnosability} {With} {Measurement} {Uncertainty}},
	volume = {61},
	issn = {1558-2523},
	doi = {10.1109/TAC.2015.2455111},
	number = {4},
	urldate = {2026-02-17},
	journal = {IEEE Transactions on Automatic Control},
	author = {Deng, Yi and D'Innocenzo, Alessandro and Di Benedetto, Maria Domenica and Di Gennaro, Stefano and Julius, A. Agung},
	year = {2016},
	pages = {982--993},
}

@article{girard_approximation_2007,
	title = {Approximation {Metrics} for {Discrete} and {Continuous} {Systems}},
	volume = {52},
	issn = {1558-2523},
	doi = {10.1109/TAC.2007.895849},
	number = {5},
	urldate = {2026-02-14},
	journal = {IEEE Transactions on Automatic Control},
	author = {Girard, Antoine and Pappas, George J.},
	year = {2007},
	pages = {782--798},
}

@article{girard_approximate_2011,
	title = {Approximate {Bisimulation}: {A} {Bridge} {Between} {Computer} {Science} and {Control} {Theory}},
	volume = {17},
	issn = {0947-3580},
	shorttitle = {Approximate {Bisimulation}},
	doi = {10.3166/ejc.17.568-578},
	number = {5},
	urldate = {2026-02-14},
	journal = {European Journal of Control},
	author = {Girard, Antoine and Pappas, George J.},
	year = {2011},
	pages = {568--578},
}

@inproceedings{dallas_control_2025,
	title = {Control {Barrier} {Functions} for {Shared} {Control} and {Vehicle} {Safety}},
	doi = {10.23919/ACC63710.2025.11107628},
	urldate = {2025-12-17},
	booktitle = {American {Control} {Conference} ({ACC})},
	author = {Dallas, James and Talbot, John and Suminaka, Makoto and Thompson, Michael and Lew, Thomas and Orosz, Gábor and Subosits, John},
	year = {2025},
	pages = {4203--4210},
}

@inproceedings{park_collision_2026,
	title = {Beyond {Collision} {Cones}: {Dynamic} {Obstacle} {Avoidance} for {Nonholonomic} {Robots} via {Dynamic} {Parabolic} {Control} {Barrier} {Functions}},
	shorttitle = {{DPCBF}},
	doi = {10.48550/arXiv.2510.01402},
	urldate = {2025-12-15},
	booktitle = {International {Conference} on {Robotics} and {Automation} ({ICRA})},
	author = {Park, Hun Kuk and Kim, Taekyung and Panagou, Dimitra},
	year = {2026},
}

@article{altarovici_general_2013,
	title = {A general {Hamilton}-{Jacobi} framework for non-linear state-constrained control problems},
	volume = {19},
	copyright = {© EDP Sciences, SMAI, 2012},
	issn = {1292-8119, 1262-3377},
	doi = {10.1051/cocv/2012011},
	language = {en},
	number = {2},
	urldate = {2025-12-01},
	journal = {ESAIM: Control, Optimisation and Calculus of Variations},
	author = {Altarovici, Albert and Bokanowski, Olivier and Zidani, Hasnaa},
	year = {2013},
	pages = {337--357},
}

@article{chen_guaranteed_2021,
	title = {Guaranteed {Obstacle} {Avoidance} for {Multi}-{Robot} {Operations} {With} {Limited} {Actuation}: {A} {Control} {Barrier} {Function} {Approach}},
	volume = {5},
	issn = {2475-1456},
	shorttitle = {Guaranteed {Obstacle} {Avoidance} for {Multi}-{Robot} {Operations} {With} {Limited} {Actuation}},
	doi = {10.1109/LCSYS.2020.3000748},
	number = {1},
	urldate = {2025-11-21},
	journal = {IEEE Control Systems Letters},
	author = {Chen, Yuxiao and Singletary, Andrew and Ames, Aaron D.},
	year = {2021},
	pages = {127--132},
}

@inproceedings{bastani_safe_2021,
	title = {Safe {Reinforcement} {Learning} with {Nonlinear} {Dynamics} via {Model} {Predictive} {Shielding}},
	doi = {10.23919/ACC50511.2021.9483182},
	urldate = {2025-08-01},
	booktitle = {American {Control} {Conference} ({ACC})},
	author = {Bastani, Osbert},
	year = {2021},
	pages = {3488--3494},
}

@inproceedings{lam_model_2010,
	title = {Model predictive contouring control},
	doi = {10.1109/CDC.2010.5717042},
	urldate = {2025-07-11},
	booktitle = {{IEEE} {Conference} on {Decision} and {Control} ({CDC})},
	author = {Lam, Denise and Manzie, Chris and Good, Malcolm},
	year = {2010},
	pages = {6137--6142},
}

@article{annaswamy_integration_2023,
	title = {Integration of {Adaptive} {Control} and {Reinforcement} {Learning} for {Real}-{Time} {Control} and {Learning}},
	volume = {68},
	issn = {1558-2523},
	doi = {10.1109/TAC.2023.3290037},
	number = {12},
	urldate = {2025-06-02},
	journal = {IEEE Transactions on Automatic Control},
	author = {Annaswamy, Anuradha M. and Guha, Anubhav and Cui, Yingnan and Tang, Sunbochen and Fisher, Peter A. and Gaudio, Joseph E.},
	year = {2023},
	pages = {7740--7755},
}

@article{firoozi_oampc_2025,
	title = {{OA}-{MPC}: {Occlusion}-{Aware} {MPC} for {Guaranteed} {Safe} {Robot} {Navigation} {With} {Unseen} {Dynamic} {Obstacles}},
	volume = {33},
	issn = {1558-0865},
	shorttitle = {{OA}-{MPC}},
	doi = {10.1109/TCST.2024.3520462},
	number = {3},
	urldate = {2025-05-18},
	journal = {IEEE Transactions on Control Systems Technology},
	author = {Firoozi, Roya and Mir, Alexandre and Camps, Gadiel Sznaier and Schwager, Mac},
	year = {2025},
	pages = {940--951},
}

@article{kim_visibilityaware_2025,
	title = {Visibility-{Aware} {RRT}* for {Safety}-{Critical} {Navigation} of {Perception}-{Limited} {Robots} in {Unknown} {Environments}},
	volume = {10},
	doi = {10.48550/arXiv.2406.07728},
	number = {5},
	urldate = {2025-02-21},
	journal = {IEEE Robotics and Automation Letters},
	author = {Kim, Taekyung and Panagou, Dimitra},
	year = {2025},
	pages = {4508--4515},
}

@inproceedings{bansal_hamiltonjacobi_2017,
	title = {Hamilton-{Jacobi} reachability: {A} brief overview and recent advances},
	shorttitle = {{HJ} {Reachability} {Overview}},
	doi = {10.1109/CDC.2017.8263977},
	urldate = {2024-11-19},
	booktitle = {{IEEE} {Conference} on {Decision} and {Control} ({CDC})},
	author = {Bansal, Somil and Chen, Mo and Herbert, Sylvia and Tomlin, Claire J.},
	year = {2017},
	pages = {2242--2253},
}

@article{agrawal_gatekeeper_2024,
	title = {gatekeeper: {Online} {Safety} {Verification} and {Control} for {Nonlinear} {Systems} in {Dynamic} {Environments}},
	volume = {40},
	issn = {1941-0468},
	shorttitle = {gatekeeper},
	doi = {10.1109/TRO.2024.3454415},
	urldate = {2024-10-08},
	journal = {IEEE Transactions on Robotics},
	author = {Agrawal, Devansh Ramgopal and Chen, Ruichang and Panagou, Dimitra},
	year = {2024},
	pages = {4358--4375},
}

@inproceedings{choi_robust_2021,
	title = {Robust {Control} {Barrier}–{Value} {Functions} for {Safety}-{Critical} {Control}},
	shorttitle = {{CBVF}},
	doi = {10.1109/CDC45484.2021.9683085},
	urldate = {2024-09-09},
	booktitle = {{IEEE} {Conference} on {Decision} and {Control} ({CDC})},
	author = {Choi, Jason J. and Lee, Donggun and Sreenath, Koushil and Tomlin, Claire J. and Herbert, Sylvia L.},
	year = {2021},
	pages = {6814--6821},
}

@article{garg_advances_2024,
	title = {Advances in the {Theory} of {Control} {Barrier} {Functions}: {Addressing} practical challenges in safe control synthesis for autonomous and robotic systems},
	volume = {57},
	issn = {1367-5788},
	doi = {10.1016/j.arcontrol.2024.100945},
	urldate = {2024-09-06},
	journal = {Annual Reviews in Control},
	author = {Garg, Kunal and Usevitch, James and Breeden, Joseph and Black, Mitchell and Agrawal, Devansh and Parwana, Hardik and Panagou, Dimitra},
	year = {2024},
	pages = {100945},
}

@inproceedings{lee_geometric_2010,
	title = {Geometric tracking control of a quadrotor {UAV} on {SE}(3)},
	doi = {10.1109/CDC.2010.5717652},
	urldate = {2024-07-10},
	booktitle = {{IEEE} {Conference} on {Decision} and {Control} ({CDC})},
	author = {Lee, Taeyoung and Leok, Melvin and McClamroch, N. Harris},
	year = {2010},
	pages = {5420--5425},
}

@inproceedings{chen_backup_2021,
	title = {Backup {Control} {Barrier} {Functions}: {Formulation} and {Comparative} {Study}},
	shorttitle = {Backup {CBF}},
	doi = {10.1109/CDC45484.2021.9683111},
	urldate = {2024-07-02},
	booktitle = {{IEEE} {Conference} on {Decision} and {Control} ({CDC})},
	author = {Chen, Yuxiao and Jankovic, Mrdjan and Santillo, Mario and Ames, Aaron D.},
	year = {2021},
	pages = {6835--6841},
}

@book{camacho_model_2004,
	title = {Model {Predictive} {Control}},
	isbn = {978-1-85233-694-3},
	language = {en},
	publisher = {Springer London},
	author = {Camacho, Eduardo F. and Bordons, Carlos},
	year = {2004},
}

@book{borrelli_predictive_2017,
	title = {Predictive {Control} for {Linear} and {Hybrid} {Systems}},
	isbn = {978-1-107-01688-0},
	language = {en},
	publisher = {Cambridge University Press},
	author = {Borrelli, Francesco and Bemporad, Alberto and Morari, Manfred},
	year = {2017},
}

@book{re_automotive_2010,
	title = {Automotive {Model} {Predictive} {Control}: {Models}, {Methods} and {Applications}},
	isbn = {978-1-84996-071-7},
	shorttitle = {Automotive {Model} {Predictive} {Control}},
	language = {en},
	publisher = {Springer},
	editor = {del Re, Luigi and Allgöwer, Frank and Glielmo, Luigi and Guardiola, Carlos and Kolmanovsky, Ilya},
	year = {2010},
}

@inproceedings{so_how_2024,
	title = {How to {Train} {Your} {Neural} {Control} {Barrier} {Function}: {Learning} {Safety} {Filters} for {Complex} {Input}-{Constrained} {Systems}},
	shorttitle = {{PNCBF}},
	language = {en},
	urldate = {2024-01-09},
	booktitle = {{IEEE} {International} {Conference} on {Robotics} and {Automation} ({ICRA})},
	author = {So, Oswin and Serlin, Zachary and Mann, Makai and Gonzales, Jake and Rutledge, Kwesi and Roy, Nicholas and Fan, Chuchu},
	year = {2024},
	pages = {11532--11539},
}

@article{zhou_raptor_2021,
	title = {{RAPTOR}: {Robust} and {Perception}-{Aware} {Trajectory} {Replanning} for {Quadrotor} {Fast} {Flight}},
	volume = {37},
	issn = {1941-0468},
	shorttitle = {{RAPTOR}},
	doi = {10.1109/TRO.2021.3071527},
	number = {6},
	urldate = {2023-10-11},
	journal = {IEEE Transactions on Robotics},
	author = {Zhou, Boyu and Pan, Jie and Gao, Fei and Shen, Shaojie},
	year = {2021},
	pages = {1992--2009},
}

@article{spielberg_neural_2022,
	title = {Neural {Network} {Model} {Predictive} {Motion} {Control} {Applied} to {Automated} {Driving} {With} {Unknown} {Friction}},
	volume = {30},
	issn = {1558-0865},
	doi = {10.1109/TCST.2021.3130225},
	number = {5},
	journal = {IEEE Transactions on Control Systems Technology},
	author = {Spielberg, Nathan A. and Brown, Matthew and Gerdes, J. Christian},
	year = {2022},
	pages = {1934--1945},
}

@inproceedings{kim_physics_2022,
	title = {Physics {Embedded} {Neural} {Network} {Vehicle} {Model} and {Applications} in {Risk}-{Aware} {Autonomous} {Driving} {Using} {Latent} {Features}},
	doi = {10.1109/IROS47612.2022.9981303},
	booktitle = {{IEEE}/{RSJ} {International} {Conference} on {Intelligent} {Robots} and {Systems} ({IROS})},
	author = {Kim, Taekyung and Lee, Hojin and Lee, Wonsuk},
	year = {2022},
	pages = {4182--4189},
}

@inproceedings{ames_control_2019,
	title = {Control {Barrier} {Functions}: {Theory} and {Applications}},
	shorttitle = {{CBF}},
	doi = {10.23919/ECC.2019.8796030},
	booktitle = {European {Control} {Conference} ({ECC})},
	author = {Ames, Aaron D. and Coogan, Samuel and Egerstedt, Magnus and Notomista, Gennaro and Sreenath, Koushil and Tabuada, Paulo},
	year = {2019},
	pages = {3420--3431},
}

@article{macenski_robot_2022,
	title = {Robot {Operating} {System} 2: {Design}, architecture, and uses in the wild},
	volume = {7},
	shorttitle = {ros2},
	doi = {10.1126/scirobotics.abm6074},
	number = {66},
	urldate = {2022-05-17},
	journal = {Science Robotics},
	author = {Macenski, Steven and Foote, Tully and Gerkey, Brian and Lalancette, Chris and Woodall, William},
	year = {2022},
	pages = {eabm6074},
}

@misc{jax2018github,
	title = {{JAX}: composable transformations of {Python}+{NumPy} programs},
	author = {Bradbury, James and Frostig, Roy and Hawkins, Peter and Johnson, Matthew James and Leary, Chris and Maclaurin, Dougal and Necula, George and Paszke, Adam and VanderPlas, Jake and Wanderman-Milne, Skye and Zhang, Qiao},
	year = {2018},
}

@article{knoedler_safety_2025,
	title = {Safety on the {Fly}: {Constructing} {Robust} {Safety} {Filters} via {Policy} {Control} {Barrier} {Functions} at {Runtime}},
	volume = {10},
	issn = {2377-3766},
	shorttitle = {{RPCBF}},
	doi = {10.1109/LRA.2025.3597847},
	number = {10},
	urldate = {2025-09-23},
	journal = {IEEE Robotics and Automation Letters},
	author = {Knoedler, Luzia and So, Oswin and Yin, Ji and Black, Mitchell and Serlin, Zachary and Tsiotras, Panagiotis and Alonso-Mora, Javier and Fan, Chuchu},
	year = {2025},
	pages = {10058--10065},
}
\end{document}